\documentclass{article}

\usepackage[main, final]{neurips_2026}

\usepackage[utf8]{inputenc} 
\usepackage[T1]{fontenc}    
\usepackage{hyperref}       
\usepackage{url}            
\usepackage{booktabs}       
\usepackage{amsfonts}       
\usepackage{nicefrac}       
\usepackage{microtype}      
\usepackage{xcolor}         

\usepackage{amsmath} 
\hypersetup{
colorlinks=true
}

\usepackage{amsmath, amssymb, amsthm, mathtools}
\usepackage{graphicx}
\usepackage{subcaption}
\usepackage{booktabs}
\usepackage{multirow}
\usepackage{array}
\usepackage{colortbl}
\usepackage{enumitem}

\providecolor{TableHeader}{RGB}{235,235,235}
\providecolor{TableGroup}{RGB}{245,245,245}
\providecolor{SecondRow}{RGB}{232,238,247}
\providecolor{BestRow}{RGB}{226,236,228}
\providecommand{\tabrowcolor}[1]{\rowcolor{#1}}
\providecommand{\theadc}[1]{\shortstack[c]{#1}}

\usepackage{hyperref}
\usepackage{cleveref}
\usepackage{algorithm}
\usepackage{algpseudocode}
\usepackage{tcolorbox}
\usepackage{tikz}
\usepackage{pgfplots}
\pgfplotsset{compat=1.18}
\usetikzlibrary{patterns}

\theoremstyle{plain}
\newtheorem{theorem}{Theorem}
\newtheorem{proposition}{Proposition}
\newtheorem{lemma}{Lemma}

\theoremstyle{definition}
\newtheorem{assumption}{Assumption}

\newtcolorbox{takeaway}{
  colback=blue!5,
  colframe=blue!50!black,
  boxrule=0.5pt,
  arc=2pt,
  left=4pt, right=4pt, top=2pt, bottom=2pt,
}

\title{RECAP: Looking Once Is Not Enough \\for Vision-Language Reasoning}

\author{%
  \textbf{Zhaolu Kang}$^{1,2}$, 
  \\\textbf{Tailong Luo}, 
  \textbf{Chenxin Li}$^{2}$, 
  \textbf{Yingjie He}$^{2}$,
  \textbf{Zhenyu Yu},
  \textbf{Fengyu Zhou},
  \textbf{Jiachen Qian},\\
  \textbf{Lei Wei}$^{2}$,
  \textbf{Shuang Chen},
  \textbf{Jiachen Li},
  \textbf{Eric Hanchen Jiang},
  \textbf{Rongchao Zhang}$^{2}$,\\
\textbf{Zhengtao Yao},
  \textbf{Hoi Leong Lee},
  \textbf{Guansu Wang},
  \textbf{Kaiyue Zhou}$^{2,\dagger}$ \\
  $^{1}$Peking University \quad $^{2}$Chengdu Minto Tech \\
  \texttt{zlkang25@stu.pku.edu.cn}\quad 
  \texttt{kyzhou@wayne.edu}
}

\begin{document}

\maketitle

\renewcommand{\thefootnote}{\fnsymbol{footnote}}
\footnotetext{$^{\dagger}$Corresponding Author.}
\renewcommand{\thefootnote}{\arabic{footnote}}


\begin{abstract}
Long chain-of-thought reasoning improves vision-language models (VLMs), but it also exposes a temporal grounding failure: as generation unfolds, models may drift from image evidence while reinforcement learning still provides only a final-answer reward. We study this gap through visual revisit, a temporally localized reactivation of image evidence during reasoning. Across three VLM backbones, revisit peaks are predictive of correctness and causally linked to performance: masking them degrades accuracy more than masking non-peaks, while injecting revisit-like peaks into failed rollouts partially restores correct answers. Control analyses show that this signal is not explained by raw attention strength, uncertainty, logit margin, or hidden-state norm. We propose RECAP, a GRPO-compatible training method that turns detached revisit traces into a credit-assignment signal. RECAP uses revisit-conditioned advantage estimation, dynamic discounting, and gated value prediction to propagate final rewards through visually important reasoning steps, without extra rewards, additional rollouts, or inference-time changes. Across Qwen2.5-VL-7B, InternVL3-8B, and LLaVA-OV-7B, RECAP consistently improves over GRPO and VPPO, yielding $+6.1$--$+6.5$ HallusionBench gains over the SFT base while preserving reasoning and text-only performance. Results on larger LoRA-tuned models, Chinese zero-shot benchmarks, visual perturbations, and human evaluation further demonstrate robust improvements in image-grounded reasoning.
\end{abstract}

%

\section{Introduction}
\label{sec:intro}

Vision-language models (VLMs) increasingly solve visual tasks through
long chains of thought: they parse diagrams, compare spatial evidence,
and decompose multi-hop questions into intermediate reasoning steps
that can span hundreds of tokens~\cite{vlm-cot-survey, mathvista}.
Reinforcement learning with verifiable answers, especially GRPO-style
post-training~\cite{grpo} and multimodal variants~\cite{vppo, vapo, tvc},
has become a practical way to improve these abilities. Yet the reward
used by such training usually arrives only at the final answer. It can
say whether the trajectory ended correctly, but not which intermediate
tokens should have returned to the image and which could safely rely on
language-internal reasoning.

This creates a temporal grounding problem. A long-CoT VLM should not
attend uniformly to the image at every step: some tokens perform local
algebra, bookkeeping, or linguistic continuation. The failure is more
specific. When a later decision again depends on visual evidence, the
model must re-ground the reasoning chain in the image; otherwise, it can
produce fluent but visually unsupported conclusions~\cite{hallusionbench}.
Existing remedies improve visual grounding in useful but mostly static
ways: VPPO reweights updates with visual signals~\cite{vppo}, TVC
re-injects visual attention at inference time~\cite{tvc}, and VAPO uses
learned visual anchors~\cite{vapo}. These approaches do not directly ask
when image evidence should re-enter a long reasoning trajectory, nor how
a final-answer reward should be propagated through those visually
important moments.

We study this missing temporal structure through \emph{visual revisit}:
a local reactivation of cross-modal attention after a stretch of
language-internal reasoning. Across three VLM backbones, correct
rollouts exhibit recurring revisit peaks, whereas failed rollouts more
often lose or mistime these reactivations (\cref{fig:teaser}a). The
pattern is not merely correlational. Masking cross-modal attention at
revisit peaks hurts HallusionBench accuracy substantially more than
masking random non-peaks, while boosting would-be revisit positions in
failed rollouts partially recovers accuracy (\cref{fig:teaser}b,
\cref{sec:exp-causal}). Control experiments further show that revisit is
not explained by raw attention mass, token entropy, logit margin, or
hidden-state norm. Finally, \cref{fig:teaser}c foreshadows a mechanistic
localization result: the revisit effect is concentrated in a small set of
attention heads rather than spread uniformly across the network.
Motivated by this evidence, we propose RECAP, a GRPO-compatible training
method that turns detached revisit traces into a credit-assignment
signal. After a rollout is sampled, RECAP threads the trace through a
training-time estimator input, a revisit-conditioned dynamic discount,
and a revisit-gated value predictor. The sampling policy, terminal
reward, rollout budget, and inference-time computation remain unchanged.

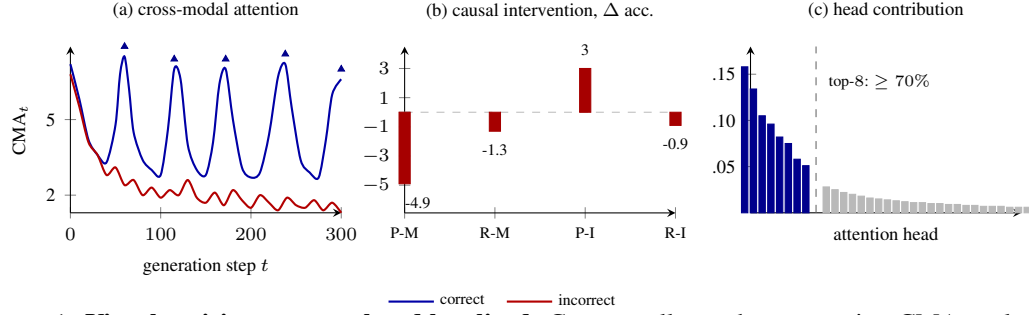
\begin{figure}[t]
  \centering
  \begin{subfigure}[t]{0.285\linewidth}
    \centering
        \begin{tikzpicture}[baseline=(current bounding box.north)]
      \begin{axis}[
        scale only axis,
        width=0.90\linewidth,
    height=2.20cm,
        xtick={0,100,200,300},
        ytick={0.02,0.05,0.08},
        yticklabels={2,5,8},
        scaled y ticks=false,
        xlabel={generation step $t$},
        ylabel={CMA$_t$},
        tick label style={font=\scriptsize},
        label style={font=\scriptsize},
                title={(a) cross-modal attention},
        title style={font=\scriptsize, yshift=0.2ex},
        legend to name=teaserlegend,
legend columns=2,
legend style={
  draw=none,
  fill=none,
  font=\tiny,
  inner sep=1pt,
  row sep=-1pt,
  /tikz/every even column/.append style={column sep=6pt},
},
        legend cell align=left,
        axis lines=left,
        clip=false,
      ]
        \addplot[blue!65!black, thick, smooth, mark=none]
          coordinates {
            (0,0.072) (10,0.058) (20,0.042) (30,0.036) (40,0.033)
            (50,0.048) (55,0.068) (60,0.075) (65,0.062) (70,0.045)
            (80,0.034) (90,0.030) (100,0.028)
            (108,0.046) (115,0.070) (122,0.066) (130,0.040)
            (140,0.030) (150,0.028) (158,0.041) (165,0.064)
            (172,0.070) (180,0.050) (190,0.031) (200,0.027)
            (210,0.029) (220,0.044) (230,0.067) (238,0.072)
            (245,0.055) (255,0.034) (265,0.029) (275,0.027)
            (283,0.042) (290,0.060) (300,0.066)
          };
        \addlegendentry{correct}
        \addplot[red!70!black, thick, smooth, mark=none]
          coordinates {
            (0,0.068) (10,0.055) (20,0.041) (30,0.036) (40,0.028)
            (50,0.031) (60,0.024) (70,0.026) (80,0.020) (90,0.023)
            (100,0.019) (110,0.022) (120,0.020) (130,0.026)
            (140,0.019) (150,0.017) (160,0.021) (170,0.016)
            (180,0.022) (190,0.018) (200,0.015) (210,0.020)
            (220,0.017) (230,0.014) (240,0.019) (250,0.016)
            (260,0.015) (270,0.018) (280,0.014) (290,0.017) (300,0.013)
          };
        \addlegendentry{incorrect}
        \addplot[only marks, mark=triangle*, mark size=1.4pt,
                 blue!55!black, forget plot]
          coordinates {(60,0.079) (115,0.074) (172,0.074)
                       (238,0.076) (300,0.070)};
      \end{axis}
    \end{tikzpicture}
  \end{subfigure}%
  \hspace{0.040\linewidth}
  \begin{subfigure}[t]{0.285\linewidth}
    \centering
        \begin{tikzpicture}[baseline=(current bounding box.north)]
      \begin{axis}[
        ybar,
        scale only axis,
        width=0.90\linewidth,
       height=2.20cm,
      bar width=4.2pt,
      enlarge x limits=0.30,
      ymin=-6.9, ymax=4.5,
      ylabel={},
        symbolic x coords={pkmask, rndmask, pkinj, rndinj},
xtick=data,
xticklabels={P-M,R-M,P-I,R-I},
x tick label style={font=\tiny, align=center},
        ytick={-5,-3,-1,1,3},
        tick label style={font=\scriptsize},
        label style={font=\scriptsize},
        title={(b) causal intervention, $\Delta$ acc.},
        title style={font=\scriptsize, yshift=0.2ex},
        extra y ticks={0}, extra y tick labels={},
        extra y tick style={grid=major,
                            grid style={dashed,gray!50}},
        axis lines=left,
        clip=false,
        nodes near coords style={font=\tiny, inner sep=1pt},
        every node near coord/.append style={font=\tiny},
      ]
\addplot[fill=red!65!black, draw=red!65!black]
  coordinates {(pkmask,-4.9) (rndmask,-1.3)
               (pkinj,3.0)   (rndinj,-0.9)};

\node[font=\tiny, anchor=north]
  at ([xshift=5pt]axis cs:pkmask,-5.20) {-4.9};

\node[font=\tiny, anchor=north]
  at (axis cs:rndmask,-1.55) {-1.3};

\node[font=\tiny, anchor=south]
  at (axis cs:pkinj,3.28) {3};

\node[font=\tiny, anchor=north]
  at (axis cs:rndinj,-1.15) {-0.9};
      \end{axis}
    \end{tikzpicture}
 \end{subfigure}%
\hspace{0.040\linewidth}
  \begin{subfigure}[t]{0.285\linewidth}
    \centering
        \begin{tikzpicture}[baseline=(current bounding box.north)]
      \begin{axis}[
        ybar,
        scale only axis,
        width=0.90\linewidth,
height=2.20cm,
bar width=2.4pt,
enlarge x limits=0.08,
ymin=0, ymax=0.18,
ylabel={},
xlabel={attention head},
        xtick=\empty,
        ytick={0.05,0.10,0.15},
        yticklabels={.05,.10,.15},
        tick label style={font=\scriptsize},
        label style={font=\scriptsize},
        title={(c) head contribution},
        title style={font=\scriptsize, yshift=0.2ex},
        axis lines=left,
clip=false,
      ]
        \addplot[fill=blue!55!black, draw=blue!55!black]
          coordinates {
            (1,0.158) (2,0.134) (3,0.105) (4,0.096)
            (5,0.082) (6,0.075) (7,0.058) (8,0.051)
          };
        \addplot[fill=gray!55, draw=gray!55]
          coordinates {
            (9,0.028)  (10,0.025) (11,0.022) (12,0.020)
            (13,0.018) (14,0.016) (15,0.015) (16,0.014)
            (17,0.013) (18,0.012) (19,0.011) (20,0.011)
            (21,0.010) (22,0.010) (23,0.009) (24,0.009)
            (25,0.008) (26,0.008) (27,0.007) (28,0.007)
            (29,0.007) (30,0.006) (31,0.006) (32,0.006)
          };
        \draw[dashed, black!60, thin]
          (axis cs:8.5,0) -- (axis cs:8.5,0.18);
        \node[font=\tiny, anchor=west] at (axis cs:8.9,0.142)
            {top-$8$: $\geq70\%$};
      \end{axis}
    \end{tikzpicture}
  \end{subfigure}
{\centering
\pgfplotslegendfromname{teaserlegend}\par
}
\vspace{-0.9em}
    \caption{\textbf{Visual revisits are causal and localized.}
  Correct rollouts show recurring CMA peaks, while incorrect rollouts
  lose visual grounding after the initial window. Masking revisit peaks
  hurts accuracy more than masking random non-peaks, and peak injection
  partially recovers failed rollouts. A small set of heads explains most
  of the causal effect.}
  \vspace{-1.5em}
  \label{fig:teaser}
\end{figure}

Two analyses clarify what the revisit trace contributes. First,
\cref{prop:expressiveness} shows that revisit-conditioned dynamic
discounting can induce retroactive credit couplings through future
revisit peaks that same-step local advantage reweighting cannot express.
Second, \cref{thm:variance-reduction} provides an oracle conditioning
analysis showing when revisit-gated value prediction can reduce
return-residual variance, while Appendix~\ref{app:thm1-proof} makes
explicit the trajectory-dependent surrogate bias introduced by post-hoc
revisit conditioning. We also
introduce RAP, a method-agnostic diagnostic that measures where positive
return mass concentrates relative to revisit peaks. Empirically, RECAP
improves HallusionBench accuracy over GRPO by $+3.5$--$+4.1$ points
across Qwen2.5-VL-7B, InternVL3-8B, and LLaVA-OV-7B, corresponding to
$+6.1$--$+6.5$ points over the SFT bases, while preserving reasoning and
text-only capability. The gains persist on a 72B LoRA-tuned model,
transfer zero-shot to Chinese benchmarks, remain robust under visual
perturbations, align with human judgments of image-groundedness, and
compose with VPPO.

Our contributions are:
\begin{itemize}[leftmargin=*, noitemsep, topsep=0pt]
  \item We identify visual revisit as a temporal grounding signal in
  long-CoT VLM reasoning, and validate it with peak-masking,
  peak-injection, signal-control, and head-localization analyses.
  \item We propose RECAP, a GRPO-compatible clipped-surrogate method
  that uses detached revisit traces for advantage estimation, dynamic
  discounting, and gated value prediction without extra rewards,
  additional rollouts, or inference-time changes.
  \item We analyze why dynamic discounting is more expressive than
  same-step local reweighting, provide an oracle conditioning analysis
  for revisit-gated value prediction, and introduce RAP as a
  method-agnostic revisit-aware training diagnostic.
  \item We demonstrate consistent improvements across three backbones,
  a larger LoRA-tuned model, cross-lingual benchmarks, visual
  perturbations, human evaluation, and composition with VPPO.
\end{itemize}

\section{Related Work}
\label{sec:related}

\paragraph{RL post-training for long-CoT VLM reasoning.}
Reinforcement learning with verifiable answers has become a standard
post-training tool for long-CoT reasoning models. GRPO~\cite{grpo}
removes the PPO critic and normalizes rewards within sampled groups,
making final-answer supervision practical at scale. Recent multimodal
variants inject visual information into this recipe: VPPO~\cite{vppo}
modulates policy updates with visual signals, VAPO~\cite{vapo} uses
learned visual anchors, and TVC~\cite{tvc} re-injects visual evidence
during generation. Other methods refine credit assignment through
per-token shaping in $\lambda$-GRPO~\cite{lambdagrpo} or structured
advantages in TreeAdv~\cite{treeadv}. These methods improve how final
rewards affect updates, but they do not explicitly model when a long
multimodal trajectory should return to the image. RECAP targets this
temporal grounding gap with detached visual revisit traces that shape
advantage estimation, dynamic discounting, and value prediction.

\paragraph{Visual hallucination and grounding during generation.}
Long-CoT VLMs can produce fluent reasoning chains that drift from the
image and yield visually unsupported conclusions. Benchmarks and
analyses such as HallusionBench, POPE, MMHalBench, and TVC document this
failure mode across recognition, spatial reasoning, and instruction
following~\cite{hallusionbench, pope, mmhalbench, tvc}. Existing
mitigations improve grounding through inference-time interventions, such
as visual re-injection in TVC~\cite{tvc} or attention-suppression
decoding in OPERA~\cite{opera}, or through preference/data-level
supervision with image-perturbed negatives~\cite{vlm-head-analysis}.
These approaches are valuable but mostly treat grounding as a static
correction or decoding-time intervention. Our focus is complementary:
we ask whether the model's own trajectory contains moments when visual
evidence should re-enter the chain, and whether these revisit moments
predict and causally support correctness.

\paragraph{Revisit-aware credit assignment and cross-modal mechanisms.}
Attention-derived statistics have been used for training and diagnosis,
but their role in long-CoT multimodal credit assignment remains
underexplored. Raw attention is hard to interpret in isolation because
attention mass varies with token type, position, and layer. RECAP
therefore defines revisit as a local rise of cross-modal attention above
a trajectory-specific baseline, marking when the model returns to visual
evidence after language-internal reasoning. Unlike same-step or
token-local reweighting, revisit-conditioned discounting can propagate
final-answer credit retroactively through future revisit peaks
(Proposition~\ref{prop:expressiveness}), while the gated value predictor
addresses the return-residual bias--variance trade-off studied in
Theorem~\ref{thm:variance-reduction} and Appendix~\ref{app:thm1-proof}.
Related mechanistic studies localize cross-modal interaction to small
head subsets~\cite{vlm-head-analysis, causal-probing}; we use this not
as the method itself, but as evidence that RECAP encourages identifiable
cross-modal mechanisms rather than diffuse attention noise
(\cref{sec:exp-mechinterp}).

%

\section{The RECAP Framework}
\label{sec:method}

\subsection{Preliminaries}
\label{sec:method-prelim}

\paragraph{GRPO for VLMs.}
A VLM defines an autoregressive policy
$\pi_\theta(y_t\mid h_t)$ with context
$h_t=(\mathbf v,q,y_{<t})$. For a prompt--image pair
$(\mathbf v,q)$, GRPO~\cite{grpo} samples a group of $M$ CoT rollouts
$\{\tau_g\}_{g=1}^{M}$, scores each rollout with a terminal verifier
reward $R(\tau_g)$, and assigns the same group-normalized advantage to
every token in the rollout:
\begin{equation}
  \hat A_t^{(g)}
  =
  \frac{R(\tau_g)-\mu_R}{\sigma_R+\epsilon},
  \qquad
  \mu_R=\tfrac{1}{M}\sum_{j=1}^{M}R(\tau_j),
  \quad
  \sigma_R^2=\tfrac{1}{M}\sum_{j=1}^{M}(R(\tau_j)-\mu_R)^2 .
\end{equation}
The clipped objective is
\begin{equation}
  \mathcal{L}^{GRPO}(\theta)
  =
  -\mathbb{E}_{\tau,t}
  \left[
    \min\!\left(
      \rho_t\hat A_t,
      \mathrm{clip}(\rho_t,1{-}\epsilon,1{+}\epsilon)\hat A_t
    \right)
  \right]
  +
  \beta_{\mathrm{KL}}\,
  \mathrm{KL}\!\left[\pi_\theta\,\|\,\pi_{\mathrm{ref}}\right],
  \label{eq:grpo}
\end{equation}
where
$\rho_t=\pi_\theta(y_t\mid h_t)/\pi_{\theta_{\mathrm{old}}}(y_t\mid h_t)$.
RECAP keeps the GRPO sampling protocol, clipped policy-ratio update,
and KL anchor, but replaces the uniform token advantage with a
revisit-conditioned surrogate advantage.

\paragraph{Cross-modal attention (CMA).}
Let $\mathcal V$ index visual tokens in the joint visual--text sequence.
For layer $\ell$ and head $h$, the attention mass from text token $y_t$
to the image is
$\mathrm{CMA}_t^{(\ell,h)}=\sum_{j\in\mathcal V}\alpha_{t,j}^{(\ell,h)}$.
We average over a fixed middle-layer set $\mathcal I$ and all heads:
\begin{equation}
  \mathrm{CMA}_t
  =
  \frac{1}{|\mathcal I|N_h}
  \sum_{\ell\in\mathcal I}\sum_{h=1}^{N_h}
  \mathrm{CMA}_t^{(\ell,h)} .
  \label{eq:cma-mean}
\end{equation}
The middle-third default avoids low-layer positional effects and
high-layer label-projection noise; layer ranges and pooling ablations
are given in Appendix~\ref{app:revisit-signal}.

\subsection{The Revisit Signal}
\label{sec:method-revisit}

Absolute CMA is a weak grounding score because some tokens require
continuous visual attention and others need none. RECAP instead uses
\emph{visual revisit}: a local positive jump in CMA after recent
language-internal reasoning. With window size $W$ (default $16$), let
$\mathcal P_t=\{k:\max(1,t-W)\le k<t\}$ and, for nonempty
$\mathcal P_t$, define
\begin{equation}
  \mathrm{baseline}_t
  =
  \frac{1}{|\mathcal P_t|}
  \sum_{k\in\mathcal P_t}\mathrm{CMA}_k,
  \qquad
  z_t
  =
  \frac{\mathrm{ReLU}(\mathrm{CMA}_t-\mathrm{baseline}_t)}
       {\mathrm{baseline}_t+\epsilon} .
  \label{eq:revisit-raw}
\end{equation}
For the first generated position, where $\mathcal P_t$ is empty, we set
$z_t=0$. We then normalize per rollout with a temperature-scaled sigmoid:
\begin{equation}
  \tilde r_t
  =
  \sigma\!\left(
    \eta
    \frac{z_t-\min_{t'}z_{t'}}
         {\max_{t'}z_{t'}-\min_{t'}z_{t'}+\epsilon}
    -\eta/2
  \right)
  \in[0,1],
  \label{eq:revisit-normalized}
\end{equation}
with default $\eta=6$. Here $z_t$ is a raw revisit signal, not a reward.
The bounded signal can be used as a gate or discount multiplier, is
validated causally in Sec.~\ref{sec:exp-causal}, and is detached before
entering the objective. Importantly, $\tilde r_t$ is computed after a
rollout is sampled from the ordinary policy $\pi_\theta(y_t\mid h_t)$;
it changes training-time advantage construction, not the online sampling
distribution. Further design choices and failure modes are in
Appendix~\ref{app:revisit-signal}.

\subsection{RECAP Objective}
\label{sec:method-recap}

\begin{figure}[t]
  \centering
    \includegraphics[width=1\linewidth, trim=15 150 10 5, clip]{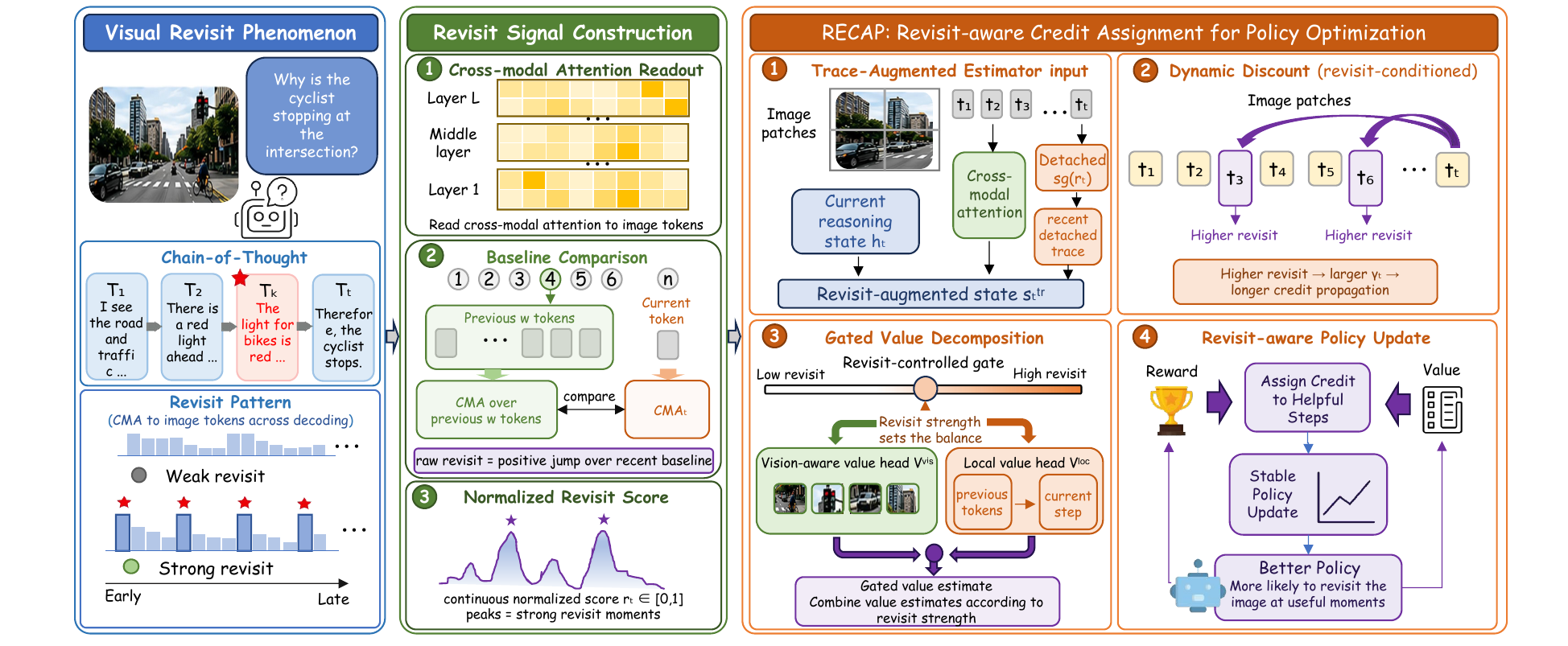}
  \vspace{-1.5em}
    \caption{The RECAP framework. A detached revisit signal \(\tilde r_t\) is
extracted from cross-modal attention and injected into three parts of the
surrogate estimator: the trace-augmented input, dynamic discount, and
revisit-gated value predictor. The rollout policy, clipped ratio, and KL
anchor remain unchanged.}
\vspace{-0.9em}
  \label{fig:framework}
\end{figure}

Given a completed rollout and its detached revisit trace, RECAP builds a
surrogate advantage from three components: a trace-augmented estimator
input, a revisit-conditioned discount, and a revisit-gated value
predictor. Appendix~\ref{app:recap-correspondence} compares these
structural slots with related baselines.

\paragraph{Trace-augmented estimator input.}
For advantage and value estimation, RECAP augments the language-state
context with recent detached revisit traces:
\begin{equation}
  s_t^{\mathrm{tr}}
  =
  \bigl(
    h_t,
    \operatorname{sg}(\tilde r_t),
    \operatorname{sg}(\tilde r_{t-1}),\ldots,
    \operatorname{sg}(\tilde r_{t-K})
  \bigr),
  \label{eq:trace-state}
\end{equation}
with missing past traces padded by zero. This is a post-hoc estimator
input, not the causal state used to sample $y_t$.

\paragraph{Revisit-conditioned dynamic discount.}
A static discount propagates terminal reward through all transitions in
the same way. RECAP instead increases the discount at visually important
positions while keeping it below one:
\begin{equation}
  \gamma_t
  =
  \gamma_{\mathrm{base}}
  +
  \beta(\gamma_{\max}-\gamma_{\mathrm{base}})
  \operatorname{sg}(\tilde r_t),
  \qquad
  0\le\beta\le1,
  \quad
  0<\gamma_{\mathrm{base}}<\gamma_{\max}<1 .
  \label{eq:dynamic-discount}
\end{equation}
We attach $\gamma_t$ to the transition into position $t$, so return at
step $t$ uses $\gamma_{t+1}$ when propagating reward from $t+1$ back to
$t$. Thus future revisit peaks can retroactively increase credit for
earlier decisions.

\paragraph{Revisit-gated value prediction.}
Low-revisit tokens are often language-internal continuations, whereas
high-revisit tokens coincide with renewed visual evidence. We use two
value heads with a shared backbone and combine them by the detached
current revisit value:
\begin{equation}
  V^{\mathrm{gated}}(s_t^{\mathrm{tr}})
  =
  (1-\operatorname{sg}(\tilde r_t))
  V^{\mathrm{loc}}(s_t^{\mathrm{tr}})
  +
  \operatorname{sg}(\tilde r_t)
  V^{\mathrm{vis}}(s_t^{\mathrm{tr}}).
  \label{eq:gated-value}
\end{equation}
Here $V^{\mathrm{loc}}$ and $V^{\mathrm{vis}}$ denote the learned
implementation heads and may use the full trace-augmented input
$s_t^{\mathrm{tr}}$. The oracle analysis in Sec.~\ref{sec:method-thm1}
uses a restricted abstraction to isolate the variance effect of
conditioning on $\tilde r_t$. This predictor is used inside a detached
policy advantage; the bias--variance interpretation is formalized in
Sec.~\ref{sec:method-thm1} and Appendix~\ref{app:thm1-proof}.

\paragraph{Return and value target.}
RECAP starts from the same raw terminal verifier score as GRPO. For a
completed rollout $\tau$ of length $T$, define
$b_t(\tau)=0$ for $t<T$ and $b_T(\tau)=R_{\mathrm{raw}}(\tau)$.
For the revisit-discounted surrogate return, RECAP propagates the
non-negative stream
\[
  u_t=\max(b_t(\tau),0).
\]
The revisit signal $\tilde r_t$ is not treated as a reward; it only
modulates the discount and value prediction. The revisit-discounted
return is
\begin{equation}
  G_t^{\mathrm{VR}}
  =
  \sum_{k=0}^{T-t}
  \left(\prod_{i=1}^{k}\gamma_{t+i}\right)
  u_{t+k} .
  \label{eq:vr-return}
\end{equation}
Both value heads are trained through the gated prediction loss
\begin{equation}
  L_V
  =
  \mathbb E
  \left[
    \left(
      \operatorname{sg}(G_t^{\mathrm{VR}})
      -
      V^{\mathrm{gated}}(s_t^{\mathrm{tr}})
    \right)^2
  \right].
  \label{eq:value-loss}
\end{equation}
The exact reward-stream convention and log-domain return recursion are
specified in Appendix~\ref{app:log-domain}.

\paragraph{Policy surrogate.}
The per-token RECAP advantage is
\begin{equation}
  A_t^{\mathrm{VR}}
  =
  G_t^{\mathrm{VR}}
  -
  V^{\mathrm{gated}}(s_t^{\mathrm{tr}}),
  \label{eq:recap-advantage}
\end{equation}
and the clipped surrogate is
\begin{equation}
  \mathcal L^{\mathrm{RECAP}}(\theta)
  =
  -\mathbb E_{\tau,t}
  \left[
    \min\!\left(
      \rho_t \operatorname{sg}(A_t^{\mathrm{VR}}),
      \mathrm{clip}(\rho_t,1-\epsilon,1+\epsilon)
      \operatorname{sg}(A_t^{\mathrm{VR}})
    \right)
  \right]
  +
  \beta_{\mathrm{KL}}
  \mathrm{KL}\!\left[\pi_\theta\,\|\,\pi_{\mathrm{ref}}\right],
  \label{eq:recap-loss}
\end{equation}
with the same causal policy ratio
$\rho_t=\pi_\theta(y_t\mid h_t)/\pi_{\theta_{\mathrm{old}}}(y_t\mid h_t)$.
Thus no policy-loss gradient flows through
$G_t^{\mathrm{VR}}$, $V^{\mathrm{gated}}$, or $\tilde r_t$; the value
heads are trained only by Eq.~\eqref{eq:value-loss}. The revisit trace
therefore affects the detached advantage, not the rollout distribution.
Vanilla GRPO is recovered by replacing $A_t^{\mathrm{VR}}$ with the
group-normalized terminal advantage $\hat A_t$.

\subsection{The RAP Metric}
\label{sec:method-rap}

Accuracy alone cannot reveal whether a policy places positive return on
visually grounded reasoning steps. We therefore report Revisit-Aware
Propagation ($\mathrm{RAP}$), the return-weighted concentration of
revisit signal. For evaluation, use a fixed method-independent return
$\bar G_t^{\mathrm{eval}}$ computed from the terminal verifier reward
with the same scalar discount for all methods. Then
\begin{equation}
  \mathrm{RAP}(\pi_\theta)
  =
  \mathbb E_{\tau\sim\pi_\theta}
  \left[
    \frac{\sum_t \tilde r_t\,\bar G_t^{\mathrm{eval},+}}
         {\sum_t \bar G_t^{\mathrm{eval},+}+\epsilon}
  \right],
  \qquad
  \bar G_t^{\mathrm{eval},+}
  =
  \max(\bar G_t^{\mathrm{eval}},0).
  \label{eq:rap}
\end{equation}
RAP is bounded, insensitive to gratuitous revisit spikes at zero-return
positions, and method-agnostic because it depends only on forward-pass
attention and verifier rewards. Appendix~\ref{app:revisit-signal}
proves these properties; Sec.~\ref{sec:exp-rap} validates its empirical
use as a training diagnostic.

\subsection{Proposition 1: A Toy Separation from Same-Step Reweighting}

Same-step local reweighting scales an advantage by $w(\tilde r_t)$ and
therefore only uses the current token's revisit value. In contrast,
dynamic discounting multiplies future discounts, so later revisit peaks
can change the credit assigned to earlier actions.

\begin{proposition}[Separation from same-step local reweighting]
\label{prop:expressiveness}
There exist a finite-horizon MDP, a revisit-pattern distribution, and
parameters $(\beta,\gamma_{\mathrm{base}},\gamma_{\max})$ such that
RECAP's revisit-conditioned dynamic discount induces a strict earlier
action preference that no strictly positive same-step local reweighting
$w(\tilde r_t)A_t$ can induce. Equivalently, in the one-decision
policy-gradient construction of Appendix~\ref{app:prop1-proof}, the
dynamic-discount fixed point is outside the union of fixed-point sets
induced by such local reweightings.
\end{proposition}

This isolates the retroactive coupling created by
$\prod_i\gamma_i$: a future revisit peak can affect earlier credit,
whereas $w(\tilde r_t)$ only rescales the current step. The horizon-3
construction, scope, and relation to trajectory-dependent reweighting
are in Appendix~\ref{app:prop1-proof}.

\subsection{Theorem 1: Revisit-Gated Return-Residual Variance}
\label{sec:method-thm1}

The following result is an oracle residual-variance analysis under
post-hoc revisit conditioning, not a policy-gradient optimality
guarantee. Let $G_t:=G_t^{\mathrm{VR}}$ and define
\[
x_t=(h_t,\operatorname{sg}(\tilde r_{t-1}),\ldots,
\operatorname{sg}(\tilde r_{t-K})),\quad
s_t^{\mathrm{tr}}=(x_t,\operatorname{sg}(\tilde r_t)),
\]
with oracle predictors
$V^0(x_t)=\mathbb E[G_t\mid x_t]$ and
$V^r(x_t,\tilde r_t)=\mathbb E[G_t\mid x_t,\tilde r_t]$.

\begin{assumption}[Revisit relevance]
\label{ass:relevance}
The revisit trace provides incremental information about the return:
$\mathbb E[G_t\mid x_t,\tilde r_t]\neq
\mathbb E[G_t\mid x_t]$ on a set of positive measure.
\end{assumption}

\begin{assumption}[Idealized linear-in-revisit oracle form]
\label{ass:value-fidelity}
For the closed-form expression below, assume that there exist oracle
component functions $\bar V^{\mathrm{loc}}(x_t)$ and
$\bar V^{\mathrm{vis}}(x_t)$ such that
\[
\mathbb E[G_t\mid x_t,\tilde r_t]
=
(1-\tilde r_t)\bar V^{\mathrm{loc}}(x_t)
+
\tilde r_t\bar V^{\mathrm{vis}}(x_t).
\]
This oracle restriction is used only for Eq.~\eqref{eq:linear-residual-variance};
the learned heads in Eq.~\eqref{eq:gated-value} take the full
$s_t^{\mathrm{tr}}$ as input and need not realize this decomposition.
\end{assumption}

\begin{theorem}[Oracle return-residual variance reduction]
\label{thm:variance-reduction}
Assume $G_t$ has finite second moment. Under
Assumption~\ref{ass:relevance},
\begin{equation}
\begin{aligned}
&
\mathrm{Var}\!\left(G_t-V^0(x_t)\right)
-
\mathrm{Var}\!\left(G_t-V^r(x_t,\tilde r_t)\right)
\\[-1mm]
&\quad =
\mathbb E_{x_t}
\left[
  \mathrm{Var}\!\left(
    \mathbb E[G_t\mid x_t,\tilde r_t]
    \,\middle|\,
    x_t
  \right)
\right]
\ge0 .
\end{aligned}
\label{eq:residual-variance-reduction}
\end{equation}
If Assumption~\ref{ass:value-fidelity} also holds, then
\begin{equation}
\begin{aligned}
&
\mathrm{Var}\!\left(G_t-V^0(x_t)\right)
-
\mathrm{Var}\!\left(G_t-V^r(x_t,\tilde r_t)\right)
\\[-1mm]
&\quad =
\mathbb E_{x_t}
\left[
  \bigl(
    \bar V^{\mathrm{vis}}(x_t)-\bar V^{\mathrm{loc}}(x_t)
  \bigr)^2
  \mathrm{Var}(\tilde r_t\mid x_t)
\right].
\end{aligned}
\label{eq:linear-residual-variance}
\end{equation}
\end{theorem}

Thus conditioning the oracle predictor on the current revisit cannot
increase return-residual variance, and the linear oracle form makes the
reduction depend on both the local/visual value gap and the conditional
variation of $\tilde r_t$. Since $\tilde r_t$ is a post-hoc trajectory
feature, RECAP remains a revisit-conditioned surrogate with an explicit
bias--variance trade-off; Appendix~\ref{app:thm1-proof} gives the proof
and derives the corresponding bias term.

%

\section{Experiments}
\label{sec:exp}

\subsection{Setup}
\label{sec:exp-setup}

We evaluate RECAP on Qwen2.5-VL-7B~\cite{qwen25vl},
InternVL3-8B~\cite{internvl3}, and LLaVA-OneVision-7B~\cite{llavaov};
a larger Qwen2.5-VL-72B LoRA run is reported in
Appendix~\ref{app:hyperparams}. All methods use the same two-stage
pipeline: Stage~1 supervised fine-tuning on 50K filtered long-CoT
examples, followed by Stage~2 RL on 30K disjoint verifiable-answer
questions (20K for 72B), with group size $M=8$ and the same terminal
verifier. RECAP's non-negative terminal stream is used only inside its
revisit-discounted surrogate return; it introduces no additional reward
model or auxiliary supervision. Hyperparameters are tuned on held-out
splits for the baselines and then frozen for RECAP; full training,
compute, and seed details are in Appendix~\ref{app:hyperparams}.

We report six benchmark groups: hallucination
(HallusionBench~\cite{hallusionbench}, POPE~\cite{pope},
MMHalBench~\cite{mmhalbench}), reasoning
(MathVista~\cite{mathvista}, MMMU~\cite{mmmu},
MathVerse~\cite{mathverse}, MM-Vet~\cite{mmvet}), grounding,
text-only capability, cross-lingual Chinese benchmarks, and adversarial
image/text perturbations. The main baselines are GRPO, VPPO, RECAP, and
a no-RL SFT+TVC-style reference; VAPO, $\lambda$-GRPO, and TreeAdv are
additionally trained on Qwen2.5-VL-7B. Headline results use 8 seeds and
ablations use 5; boldface marks the best value within each comparison
block, and daggers mark RECAP-over-GRPO significance at $p<0.05$ after
Holm--Bonferroni correction.

\subsection{Main Results}
\label{sec:exp-main}

\Cref{tab:main-3backbones} gives the main result. RECAP improves
HallusionBench accuracy over GRPO by $+3.8$, $+3.5$, and $+4.1$ points
on Qwen2.5-VL-7B, InternVL3-8B, and LLaVA-OneVision-7B, respectively,
while preserving reasoning and text-only capability. VPPO yields smaller
hallucination gains and slightly weaker reasoning scores, while the
no-RL TVC-style reference remains below RECAP. The same pattern persists
on the Qwen2.5-VL-72B LoRA setting, where RECAP improves
HallusionBench by $+3.1$ points over LoRA-GRPO with no inference-time
cost (Appendix~\ref{app:hyperparams}).

\begin{table}[t]
  \centering
  \scriptsize
  \setlength{\tabcolsep}{2.2pt}
  \renewcommand{\arraystretch}{1.10}
  \resizebox{\linewidth}{!}{%
  \begin{tabular}{l ccc cccc ccc cc c}
    \toprule
    \tabrowcolor{TableHeader}
    \textbf{Method} &
        \multicolumn{3}{c}{\textbf{Hallucination}} &
        \multicolumn{4}{c}{\textbf{Reasoning}} &
        \multicolumn{3}{c}{\textbf{Grounding}} &
        \multicolumn{2}{c}{\textbf{Text-only}} & \textbf{$\Delta$} \\
    \tabrowcolor{TableHeader}
        & \theadc{Hall.\\$\uparrow$} & \theadc{POPE\\$\uparrow$} & \theadc{MMHal\\$\uparrow$} &
         \theadc{MathV.\\$\uparrow$} & \theadc{MMMU\\$\uparrow$} & \theadc{MathVr.\\$\uparrow$} & \theadc{MM-Vet\\$\uparrow$} &
         \theadc{GQA\\$\uparrow$} & \theadc{BLINK\\$\uparrow$} & \theadc{MMVP\\$\uparrow$} &
         \theadc{MMLU\\$\uparrow$} & \theadc{GSM8K\\$\uparrow$} & \\
    \midrule
    \tabrowcolor{TableGroup}
    \multicolumn{14}{l}{\textbf{Qwen2.5-VL-7B}} \\
SFT base     & 60.1 & 86.5 & 3.12 & 59.4 & 41.3 & 33.2 & 51.4 & 63.8 & 48.2 & 54.0 & \textbf{72.4} & \textbf{78.9} & -- \\
TVC (no RL)  & 61.4 & 86.9 & 3.18 & 59.2 & 41.0 & 33.5 & 51.8 & 64.1 & 48.5 & 54.3 & 72.3 & 78.7 & $+1.3$ \\
GRPO         & 62.8 & 87.1 & 3.27 & $\underline{61.2}$ & $\underline{42.5}$ & $\underline{34.7}$ & $\underline{52.8}$ & 64.3 & $\underline{49.1}$ & 54.8 & $\underline{72.1}$ & $\underline{78.5}$ & $+2.7$ \\
VPPO         & $\underline{64.2}$ & $\underline{87.6}$ & $\underline{3.35}$ & 60.8 & 41.9 & 34.1 & 52.1 & $\underline{64.7}$ & 48.9 & $\underline{55.1}$ & 71.8 & 77.9 & $\underline{+4.1}$ \\
\tabrowcolor{BestRow}
\textbf{RECAP} & \textbf{66.6}$^{\dagger}$ & \textbf{88.4} & \textbf{3.58} & \textbf{63.6} & \textbf{42.7} & \textbf{36.1} & \textbf{53.9} & \textbf{65.2} & \textbf{50.1} & \textbf{56.4} & 72.3 & 78.6 & \textbf{+6.5}$^{\dagger}$ \\
    \midrule
    \tabrowcolor{TableGroup}
    \multicolumn{14}{l}{\textbf{InternVL3-8B}} \\
SFT base     & 58.6 & 87.2 & 3.21 & 61.7 & 44.8 & 35.0 & 53.1 & 65.0 & 49.6 & 53.2 & \textbf{73.8} & \textbf{80.1} & -- \\
TVC (no RL)  & 60.0 & 87.5 & 3.28 & 61.5 & 44.5 & 35.3 & 53.4 & 65.3 & 49.9 & 53.6 & 73.6 & 79.9 & $+1.4$ \\
GRPO         & 61.2 & 87.8 & 3.34 & $\underline{63.1}$ & $\underline{45.6}$ & $\underline{36.2}$ & $\underline{54.2}$ & 65.6 & $\underline{50.2}$ & 54.0 & $\underline{73.4}$ & $\underline{79.6}$ & $+2.6$ \\
VPPO         & $\underline{62.4}$ & $\underline{88.1}$ & $\underline{3.41}$ & 62.6 & 44.9 & 35.7 & 53.7 & $\underline{65.9}$ & 50.0 & $\underline{54.3}$ & 73.1 & 79.1 & $\underline{+3.8}$ \\
\tabrowcolor{BestRow}
\textbf{RECAP} & \textbf{64.7}$^{\dagger}$ & \textbf{89.0} & \textbf{3.62} & \textbf{65.3} & \textbf{45.9} & \textbf{37.4} & \textbf{55.1} & \textbf{66.5} & \textbf{51.1} & \textbf{55.5} & 73.5 & 79.7 & \textbf{+6.1}$^{\dagger}$ \\
    \midrule
    \tabrowcolor{TableGroup}
    \multicolumn{14}{l}{\textbf{LLaVA-OV-7B}} \\
SFT base     & 56.9 & 85.4 & 3.05 & 56.2 & 39.8 & 31.4 & 49.6 & 62.5 & 46.8 & 52.1 & \textbf{70.2} & \textbf{75.4} & -- \\
TVC (no RL)  & 58.5 & 85.8 & 3.11 & 56.0 & 39.6 & 31.7 & 49.9 & 62.8 & 47.0 & 52.4 & 70.0 & 75.2 & $+1.6$ \\
GRPO         & 59.3 & 86.1 & 3.19 & $\underline{58.0}$ & $\underline{40.9}$ & $\underline{32.8}$ & $\underline{50.8}$ & 63.0 & $\underline{47.6}$ & 52.9 & $\underline{69.7}$ & $\underline{74.9}$ & $+2.4$ \\
VPPO         & $\underline{60.6}$ & $\underline{86.5}$ & $\underline{3.25}$ & 57.4 & 40.2 & 32.2 & 50.2 & $\underline{63.3}$ & 47.4 & $\underline{53.1}$ & 69.4 & 74.3 & $\underline{+3.7}$ \\
\tabrowcolor{BestRow}
\textbf{RECAP} & \textbf{63.4}$^{\dagger}$ & \textbf{87.4} & \textbf{3.50} & \textbf{60.5} & \textbf{41.1} & \textbf{34.3} & \textbf{51.9} & \textbf{63.9} & \textbf{48.6} & \textbf{54.3} & 69.9 & 75.0 & \textbf{+6.5}$^{\dagger}$ \\
    \bottomrule
  \end{tabular}%
  }
  \vspace{2pt}
    \caption{Main results across three VLM backbones. Values are mean
benchmark scores or accuracies over $8$ seeds; MMHalBench is reported
on its native score scale. Per-seed logs, standard deviations, standard
errors, and 95\% confidence intervals are reported in the supplementary
reproducibility logs. Bold marks the best value within each
  backbone block; underlines mark the strongest non-RECAP RL baseline
  per column. A dagger marks a significant RECAP-over-GRPO improvement
  after Holm--Bonferroni correction. $\Delta$ denotes the HallusionBench
  gain over the per-backbone SFT base.}
  \vspace{-2em}
  \label{tab:main-3backbones}
\end{table}

\subsection{Ablations and Signal Characterization}
\label{sec:exp-ablations}
\label{sec:exp-signal-indep}

\Cref{tab:ablations-combined} condenses the key ablations on
Qwen2.5-VL-7B. Composing RECAP with VPPO gives an additional gain,
consistent with \cref{prop:expressiveness}: dynamic discounting and
same-step visual reweighting address different slots in the update.
Removing the dynamic discount causes the largest drop, followed by the
gated value predictor and the trace-augmented estimator input; the
variance increase without the gated value matches the trend predicted by
\cref{thm:variance-reduction}. Using the same revisit signal as a
local advantage reweight is also weaker than using it as a dynamic
discount.

\begin{table}[t]
  \centering
  \scriptsize
  \setlength{\tabcolsep}{3pt}
  \renewcommand{\arraystretch}{1.10}
  \begin{tabular}{l ccc c}
    \toprule
    \tabrowcolor{TableHeader}
    \textbf{Configuration} & \theadc{Hallusion\\$\uparrow$} & \theadc{MathVista\\$\uparrow$} & \theadc{MMMU\\$\uparrow$} & \theadc{$\mathrm{Var}(A_t)$\\rel. $\downarrow$} \\
    \midrule
    \tabrowcolor{TableGroup}
    \multicolumn{5}{l}{\textbf{(a) VPPO orthogonality}} \\
    GRPO                        & 62.8 [62.0, 63.5] & 61.2 & 42.5 & $1.81\times$ \\
    VPPO                        & 64.2 [63.3, 65.0] & 60.8 & 41.9 & $1.35\times$ \\
    RECAP                       & $\underline{66.6}$ [65.7, 67.4] & $\underline{63.6}$ & \textbf{42.7} & $1.00\times$ \\
    \tabrowcolor{BestRow}
    \textbf{VPPO~+~RECAP (composition)}  & \textbf{68.0} [67.1, 68.8] & \textbf{63.9} & $\underline{42.6}$ & \textbf{$0.98\times$} \\
    \midrule
    \tabrowcolor{TableGroup}
    \multicolumn{5}{l}{\textbf{(b) Component ablations}} \\
    \tabrowcolor{BestRow}
    \textbf{Full RECAP}                                      & \textbf{66.6} & \textbf{63.6} & \textbf{42.7} & \textbf{$1.00\times$} \\
    $-$ Trace-augmented estimator state ($K=0$)            & $\underline{65.8}$ & $\underline{63.1}$ & $\underline{42.4}$ & $1.08\times$ \\
    $-$ Gated value ($V^{\mathrm{loc}}=V^{\mathrm{vis}}$) & 65.2 & 62.8 & 42.1 & $1.70\times$ \\
    $-$ Dynamic discount ($\beta=0$)        & 64.3 & 62.4 & 42.0 & $1.09\times$ \\
    $-$ Stop-gradient on $\tilde r_t$       & 60.8 & 59.7 & 41.6 & $2.35\times$ \\
    \midrule
    \tabrowcolor{TableGroup}
    \multicolumn{5}{l}{\textbf{(c) Dynamic discount vs. advantage reweight}} \\
    Reweight $w(\tilde r_t) \cdot A_t$ (tuned $w$, Qwen/Intern/LLaVA)
    & 64.6 & 62.8 & 42.3 & $1.21\times$ \\
    \tabrowcolor{BestRow}
Dynamic discount $\gamma_t$ (ours, Qwen/Intern/LLaVA)
    & \textbf{66.6} & \textbf{63.6} & \textbf{42.7} & \textbf{$1.00\times$} \\
\multicolumn{5}{l}{\quad Gap (discount $-$ reweight), Qwen/Intern/LLaVA: $+2.0$ / $+1.6$ / $+2.1$\,pt on Hallusion.} \\
    \bottomrule
  \end{tabular}
  \vspace{2pt}
  \caption{Ablations on Qwen2.5-VL-7B. Bold marks the best value within
  each panel and underlines mark the strongest non-best variant. The
  table compares VPPO composition, component removal, and dynamic-discount
  reweighting variants across accuracy and surrogate-advantage variance.}
   \vspace{-2em}
  \label{tab:ablations-combined}
  \label{tab:orthogonality}           
  \label{tab:component-ablations}     
  \label{tab:discount-vs-reweight}    
\end{table}

\Cref{tab:signal-and-norl} tests whether the effect is specific to the
revisit signal. Replacing $\tilde r_t$ with entropy, logit margin, or
hidden-state norm under the same RECAP framework underperforms; using
revisit only at inference time without RL also fails to recover the full
gain. The positive correlation between revisit and $|A_t|$ supports the
relevance condition in \cref{thm:variance-reduction}. Sensitivity to
layer range, window size, and normalization temperature is reported in
Appendix~\ref{app:hyperparams}.

\begin{table}[t]
  \centering
  \scriptsize
  \setlength{\tabcolsep}{3.5pt}
  \renewcommand{\arraystretch}{1.10}
  \begin{tabular}{l ccc c}
    \toprule
    \tabrowcolor{TableHeader}
    \textbf{Configuration} & \theadc{Qwen2.5-VL\\7B $\uparrow$} & \theadc{InternVL3\\8B $\uparrow$} & \theadc{LLaVA-OV\\7B $\uparrow$} & \textbf{Extra} \\
    \midrule
    \tabrowcolor{TableGroup}
    \multicolumn{5}{l}{\textbf{(a) Signal identity}} \\
    Token entropy                   & 63.2 & 61.6 & 60.1 & $\rho{=}0.29$ \\
    Logit margin                    & 63.0 & 61.4 & 59.8 & $\rho{=}0.24$ \\
    Hidden-state norm               & $\underline{63.5}$ & $\underline{61.9}$ & $\underline{60.4}$ & $\rho{=}0.36$ \\
    \tabrowcolor{BestRow}
    \textbf{Revisit $\tilde r_t$ (ours)}     & \textbf{66.6} & \textbf{64.7} & \textbf{63.4} & \textbf{$\rho{=}1.00$} \\
    \midrule
    \tabrowcolor{TableGroup}
    \multicolumn{5}{l}{\textbf{(b) Same signal without RL}} \\
    SFT base (no RL, no revisit)                & 60.1 & 58.6 & 56.9 & $\Delta_\text{SFT}{=}0$ \\
    SFT + TVC                 & 61.4 & 60.0 & 58.5 & $+1.4$ \\
SFT + revisit re-injection & $\underline{62.7}$ & $\underline{61.4}$ & $\underline{59.9}$ & $\underline{+2.8}$ \\
    GRPO (RL, no revisit)                       & 62.8 & 61.2 & 59.3 & $+2.5$ \\
    \tabrowcolor{BestRow}
    \textbf{RECAP (RL + revisit)}                        & \textbf{66.6} & \textbf{64.7} & \textbf{63.4} & \textbf{$+6.4$} \\
    \midrule
    $\Delta$ RECAP vs. best no-RL baseline     & $+3.9$ & $+3.3$ & $+3.5$ & -- \\
    \bottomrule
  \end{tabular}
  \vspace{2pt}
  \caption{Signal identity and no-RL controls. Bold marks the best value
  within each panel; underlines mark the strongest signal-replacement
  baseline in panel (a) and the strongest no-RL baseline in panel (b).
  Replacing $\tilde r_t$ with surface token signals weakens RECAP, and
  using the revisit signal without RL does not recover the full gain.}
 \vspace{-2em}
  \label{tab:signal-and-norl}
  \label{tab:signal-independence}   
  \label{tab:wo-rl}                 
\end{table}

\subsection{Causal and Mechanistic Evidence}
\label{sec:exp-causal-mech}
\label{sec:exp-causal}
\label{sec:exp-mechinterp}

\Cref{fig:head-combined} summarizes the causal and mechanistic evidence.
At the token-position level, masking top revisit positions hurts far
more than masking random non-peaks, CMA injection partially recovers
failed rollouts, and revisit peaks align with annotated decision steps
above chance. At the head level, the top-10 heads carry more than
$40\%$ of revisit mass across backbones, and zeroing these heads
reproduces most of the accuracy drop from masking cross-modal attention.
Full per-backbone intervention tables, head maps, pre/post-training
concentration curves, and linear-probe analyses are in
Appendix~\ref{app:mech-interp}.

\begin{figure}[t]
  \centering
  \begin{subfigure}[c]{0.59\linewidth}
    \centering
    \setlength{\tabcolsep}{3.0pt}
    \renewcommand{\arraystretch}{1.15}
    \footnotesize
    \begin{tabular}{@{}l cccc@{}}
      \toprule
      \tabrowcolor{TableHeader}
      \textbf{Setting} & \theadc{Qwen} & \theadc{Intern} & \theadc{LLaVA} & \theadc{Mean} \\
      \midrule
      \tabrowcolor{TableGroup}
      \multicolumn{5}{@{}l}{\textbf{Accuracy change under intervention}} \\
      \tabrowcolor{BestRow}
      Peak mask        & \textbf{$-4.9$} & \textbf{$-4.6$} & \textbf{$-4.8$} & \textbf{$-4.8$} \\
      Random mask      & $-1.3$ & $-1.0$ & $-1.2$ & $-1.2$ \\
      \tabrowcolor{BestRow}
      Peak inject      & \textbf{$+3.0$} & \textbf{$+2.7$} & \textbf{$+3.1$} & \textbf{$+2.9$} \\
      Random inject    & $-0.9$ & $-0.4$ & $-0.6$ & $-0.6$ \\
      \midrule
      \tabrowcolor{TableGroup}
      \multicolumn{5}{@{}l}{\textbf{Alignment with decision steps}} \\
      \tabrowcolor{BestRow}
      $\tilde r_t$ peaks & \textbf{$42.1\%$} & \textbf{$60.8\%$} & \textbf{$61.2\%$} & \textbf{$54.7\%$} \\
      Chance baseline   & $17.2\%$ & $15.8\%$ & $19.4\%$ & $17.4\%$ \\
      \bottomrule
    \end{tabular}
    \caption{Position-level causal evidence.}
    \label{tab:causal-intervention-alignment}
    \label{fig:causal-combined}
  \end{subfigure}\hfill
  \begin{minipage}[c]{0.40\linewidth}
    \centering
    \begin{subfigure}[t]{\linewidth}
      \centering
      \caption{Head concentration.}
       \vspace{-0.5em}
      \label{fig:head-heatmap}
      \begin{tikzpicture}
        \begin{axis}[
          width=\linewidth,
          height=2.95cm,
          xlabel={top-$k$ heads},
          ylabel={revisit share (\%)},
          axis x line*=top,
          axis y line*=left,
          xmode=log,
          log basis x=10,
          xmin=3, xmax=60,
          ymin=25, ymax=100,
          xtick={5,10,20,50},
          xticklabels={5,10,20,50},
          ytick={30,50,70,90},
          grid=major,
          grid style={dashed,gray!30},
          legend style={
            at={(0.98,0.04)},
            anchor=south east,
            draw=none,
            fill=white,
            fill opacity=0.82,
            text opacity=1,
            font=\fontsize{4.8}{5.1}\selectfont,
            inner sep=0.6pt,
            row sep=-1.3pt,
            legend columns=2,
          },
          legend image post style={scale=0.5},
          legend cell align={left},
          tick label style={font=\scriptsize},
          label style={font=\scriptsize},
          mark size=1.3pt,
        ]
          \addplot[color=blue!70!black, mark=*, thick] coordinates
            {(5,34) (10,54) (20,73) (50,90)};
          \addplot[color=orange!80!black, mark=square*, thick] coordinates
            {(5,33) (10,53) (20,69) (50,86)};
          \addplot[color=green!55!black, mark=triangle*, thick] coordinates
            {(5,38) (10,57) (20,71) (50,88)};
          \addplot[color=gray!60!black, mark=none, densely dashed, thick,
            domain=3:60, samples=25] {4 * ln(x)/ln(10) * 10 + 10};
          \legend{Qwen, Intern, LLaVA, uniform}
        \end{axis}
      \end{tikzpicture}
    \end{subfigure}

    \vspace{1.0mm}
    \begin{subfigure}[t]{\linewidth}
      \centering
      \begin{tikzpicture}
        \begin{axis}[
          width=\linewidth,
          height=2.95cm,
          xlabel={top heads zeroed},
          ylabel={Hallusion $\Delta$ acc.},
          xmin=0, xmax=24,
          ymin=-9, ymax=0.5,
          xtick={0,3,5,10,20},
          ytick={0,-2,-4,-6,-8},
          grid=major,
          grid style={dashed,gray!30},
          legend style={
            at={(0.98,0.98)},
            anchor=north east,
            draw=none,
            fill=white,
            fill opacity=0.82,
            text opacity=1,
            font=\fontsize{4.8}{5.1}\selectfont,
            inner sep=0.6pt,
            row sep=-1.3pt,
          },
          legend image post style={scale=0.5},
          legend cell align={left},
          tick label style={font=\scriptsize},
          label style={font=\scriptsize},
          mark size=1.3pt,
        ]
          \addplot[color=blue!70!black, mark=*, thick] coordinates
            {(0,0) (3,-2.2) (5,-3.3) (10,-5.6) (20,-7.1)};
          \addplot[color=orange!80!black, mark=square*, thick] coordinates
            {(0,0) (3,-1.9) (5,-3.4) (10,-5.5) (20,-6.4)};
          \addplot[color=green!55!black, mark=triangle*, thick] coordinates
            {(0,0) (3,-2.4) (5,-3.2) (10,-5.9) (20,-7.5)};
          \addplot[color=red!60!black, densely dashed, thick] coordinates
            {(0,-7.6) (24,-7.6)};
          \node[anchor=west, font=\tiny, red!60!black]
            at (axis cs:0.5,-7.2) {zero all CMA: $-7.6$};
          \legend{Qwen, Intern, LLaVA}
        \end{axis}
      \end{tikzpicture}
      \vspace{-1em}
      \caption{Top-head zero-out.}
       \vspace{-0.6em}
      \label{fig:head-intervention}
    \end{subfigure}
  \end{minipage}
  \caption{\textbf{Causal and mechanistic evidence.}
  Top revisit positions are more decision-relevant than random positions,
  and revisit peaks align with annotated decision steps. The revisit
  signal further concentrates in a small set of heads, and zeroing those
  heads reproduces most of the accuracy drop from masking cross-modal
  attention.}
   \vspace{-1em}
  \label{fig:head-combined}
\end{figure}
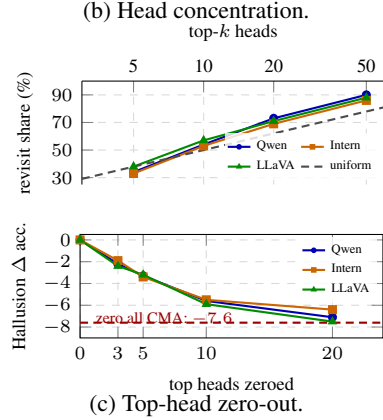

\subsection{Generalization, Robustness, and Diagnostics}
\label{sec:exp-generalization}

RECAP's gains are not confined to the three small English backbones. The
72B LoRA result, human study, cross-lingual evaluation, adversarial
tests, and RAP diagnostics are summarized here; full details are in
Appendix~\ref{app:hyperparams}, Appendix~\ref{app:human-study},
Appendix~\ref{app:cross-lingual}, Appendix~\ref{app:adversarial}, and
Appendix~\ref{app:rap-validity}, respectively.

\paragraph{Human preference.}
\label{sec:exp-human}
On 100 HallusionBench/MathVista questions with non-identical CoTs, ten
trained annotators rank image-groundedness, coherence, and
trustworthiness. RECAP wins over 60\% of image-groundedness comparisons
against GRPO and VPPO at Fleiss' $\kappa \geq 0.5$, with McNemar's test
giving $p < 0.01$ against both baselines. The LLM-as-judge ranking
agrees with the human majority on 72\% of items.

\paragraph{Cross-lingual transfer.}
\label{sec:exp-crosslingual}
Training on English data only, RECAP consistently improves over GRPO
when the same checkpoints are evaluated zero-shot on Chinese CMMMU,
MMMU, and MathVista-CN. The gains are smaller than on English
hallucination benchmarks, but remain positive without
translation-specific tuning or Chinese RL data.

\paragraph{Adversarial robustness.}
\label{sec:exp-adversarial}
Under Gaussian noise, JPEG compression, patch dropout, off-topic
distractor patches, and GPT-4 question reformulations, RECAP keeps a
larger margin than GRPO and VPPO. Strong patch dropout is the hardest
case because it disrupts the visual evidence tracked by revisit peaks,
but the margin does not invert. Text-side reformulations show lower
variation across paraphrases, consistent with using cross-modal
attention rather than surface wording.

\paragraph{RAP diagnostic.}
\label{sec:exp-rap}
RAP rises early during Stage-2 RL and plateaus higher under RECAP than
under GRPO or VPPO, matching the intended revisit-aware credit
assignment. Cross-method scatter plots further support RAP as a
method-agnostic training diagnostic rather than a RECAP-specific
objective trace.


\section{Conclusion}
\label{sec:conclusion}

We introduced visual revisit as a temporal grounding signal for
long-chain vision-language reasoning: rather than treating visual
attention as static, revisit captures when a model returns to image
evidence after language-internal reasoning. Across three VLM backbones,
revisit peaks are predictive of correctness, robust to signal controls,
and supported by masking and injection interventions. Building on this
observation, RECAP uses detached revisit traces for credit assignment
through trace-augmented estimation, dynamic discounting, and gated value
prediction. It improves visual grounding over GRPO and VPPO while
largely preserving reasoning and text-only capability, with ablations
and mechanistic analyses indicating that the gains come from assigning
credit to visually important reasoning moments rather than simply
increasing overall image attention. RECAP is also easy to insert into
verifier-based RL pipelines: the revisit trace is extracted from the
same forward pass, detached before optimization, and discarded at
inference.

\paragraph{Limitations.}
RECAP relies on accessible and meaningful cross-modal attention traces,
so its benefits may be smaller for short answers, text-only tasks, or
architectures where attention is a weak proxy for grounding. Revisit can
also be misled by salient distractors or corrupted visual regions,
suggesting future work on richer region-level or architecture-agnostic
forms of temporal grounding.

\bibliographystyle{unsrt}
\bibliography{refs}

\appendix
\renewcommand{\thesection}{\Alph{section}}
\setcounter{section}{0}


\section{Proof of Proposition~\ref{prop:expressiveness}}
\label{app:prop1-proof}

We prove a separation between revisit-conditioned dynamic discounting and
same-step local revisit reweighting. The result is not a dominance claim over
all possible trajectory-dependent reweighting schemes; rather, it isolates the
retroactive coupling induced by discount products.

\subsection{Setup}
\label{app:prop1-setup}

We work in a finite-horizon discrete-action MDP with horizon $H$.
A trajectory $\tau = (s_1, y_1, s_2, y_2, \ldots, s_H, y_H)$ is
generated by $\pi_\theta(y_t \mid s_t)$. The state $s_t$ at step $t$
encodes the prompt, the image, and the history of generated tokens
$y_{<t}$. At each step the policy receives reward $r_t = r(s_t, y_t)$.

A revisit sequence $\tilde{\mathbf r} =
(\tilde{r}_1, \ldots, \tilde{r}_H) \in [0,1]^H$ is associated
with each trajectory, computed deterministically from its cross-modal
attention pattern.\footnote{For the proof we treat $\tilde{r}_t$ as an exogenous
function of $\tau$. In practice $\tilde{r}_t$ is computed from the forward-pass
attention of $\pi_\theta$ and is therefore a function of $\theta$
through the rollout distribution; this dependence does not affect the
expressiveness-gap argument because the counterexample below isolates a
single decision point conditional on a fixed downstream revisit
pattern.}

\paragraph{Two advantage families.}
We compare two modifications of the standard discounted advantage
\[
A_t
=
\sum_{k \geq 0} \gamma^{k} r_{t+k}
-
V(s_t).
\]
The first is a same-step revisit reweighting term, and the second is a
RECAP-style dynamic-discount term:
\begin{flalign}
& A_t^{\mathrm{VR,rw}}(w)
  = w(\tilde r_t) A_t,
&& w : [0,1] \to \mathbb{R}_{>0},
\\
& A_t^{\mathrm{VR,disc}}(\beta, \gamma_{\mathrm{base}}, \gamma_{\max})
  =
  \sum_{k \geq 0}
  \Bigl(\prod_{i=1}^{k} \gamma_{t+i}\Bigr) r_{t+k}
  - V(s_t),
&& \notag\\[-0.2em]
& \gamma_i
  =
  \gamma_{\mathrm{base}}
  +
  \beta(\gamma_{\max}-\gamma_{\mathrm{base}})\tilde r_i,
&&
  0 \leq \beta \leq 1,\quad
  0 < \gamma_{\mathrm{base}} < \gamma_{\max} < 1.
\end{flalign}

For a parameterized policy class $\{\pi_\theta : \theta \in \Theta\}$,
each advantage family induces a policy gradient
$g(\theta) = \mathbb{E}_\tau\bigl[\sum_t A_t^{\mathrm{VR}} \cdot \nabla_\theta \log \pi_\theta(y_t \mid s_t)\bigr]$
and a corresponding fixed-point set
$\Pi^*_{\bullet} = \{\pi_\theta : g_\bullet(\theta) = 0\}$.
Below we exhibit a single setting -- a horizon-3 MDP with a specific
revisit-pattern distribution and specific
$(\beta, \gamma_{\mathrm{base}}, \gamma_{\max})$ -- in which the
discount-optimal policy at the sole decision step is provably outside
$\bigcup_{w} \Pi^*_{\mathrm{rw}}(w)$. This establishes the
expressiveness-gap statement of
Proposition~\ref{prop:expressiveness}.

\subsection{Horizon-3 counterexample}
\label{app:prop1-counterexample}

We exhibit an MDP, a revisit-pattern distribution, and parameters
$(\beta, \gamma_{\mathrm{base}}, \gamma_{\max})$ such that the unique
$\Pi_{\mathrm{disc}}$-optimum at the relevant decision point is not
contained in $\Pi^*_{\mathrm{rw}}(w)$ for \emph{any} choice of
$w \colon [0,1] \to \mathbb{R}_{> 0}$.

\paragraph{The MDP.}
$H = 3$. At step~1 the policy chooses one of two actions,
$\{a_{+}, a_{-}\}$. Steps~2 and~3 are forced, so $\pi_\theta$ is
non-trivial only at step~1. Let $c > 1$ be a constant.
\begin{itemize}[leftmargin=2em, itemsep=0.1em, topsep=0.1em]
  \item Action $a_{+}$ at step~1 yields rewards
  $(r_1, r_2, r_3) = (0, 0, cR)$.
  \item Action $a_{-}$ at step~1 yields rewards
  $(r_1, r_2, r_3) = (0, R, 0)$.
\end{itemize}
Thus $a_{+}$ defers a slightly larger reward to the terminal step,
whereas $a_{-}$ collects a smaller reward earlier.

\paragraph{The revisit pattern.}
The revisit sequence is deterministic:
\[
  \tilde{\mathbf r} = (0, 0, 1).
\]
Thus the sole decision step has $\tilde r_1 = 0$, while the delayed
reward of $a_{+}$ is preceded by a high-revisit step.

\paragraph{Reweight cannot change the static-discount preference.}
Under any local reweight scheme $w$, the expected advantage at step~1
for action $a$ is
\[
  A_1^{\mathrm{VR,rw}}(w)(a)
  =
  w(\tilde r_1) A_1(a)
  =
  w(0) A_1(a).
\]
Since $w(0) > 0$, reweighting cannot change the sign of the preference
between $a_{+}$ and $a_{-}$. Under the static discount
$\gamma_{\mathrm{base}}$, the two returns are
\[
  A_1(a_{+}) = \gamma_{\mathrm{base}}^2 cR,
  \qquad
  A_1(a_{-}) = \gamma_{\mathrm{base}} R.
\]
Therefore every local reweight scheme prefers $a_{-}$ whenever
\[
  \gamma_{\mathrm{base}}^2 cR
  <
  \gamma_{\mathrm{base}} R
  \quad\Longleftrightarrow\quad
  c < \frac{1}{\gamma_{\mathrm{base}}}.
\]
For such $c$,
\begin{equation}
  \pi^* \in \Pi_{\mathrm{rw}}
  \;\Longrightarrow\;
  \pi^*(a_{+} \mid s_1) = 0.
  \label{eq:rw-prefers-aminus}
\end{equation}

\paragraph{Dynamic discount can prefer $a_{+}$ while staying below one.}
Under RECAP's bounded dynamic discount,
\[
  \gamma_2 = \gamma_{\mathrm{base}},
  \qquad
  \gamma_3 = \gamma_{\max},
\]
because $\tilde r_2 = 0$ and $\tilde r_3 = 1$ with $\beta = 1$.
The dynamic-discount returns are
\[
  A_1^{\mathrm{disc}}(a_{+})
  =
  \gamma_{\mathrm{base}}\gamma_{\max} cR,
  \qquad
  A_1^{\mathrm{disc}}(a_{-})
  =
  \gamma_{\mathrm{base}} R.
\]
Thus the dynamic discount prefers $a_{+}$ whenever
\[
  \gamma_{\mathrm{base}}\gamma_{\max} cR
  >
  \gamma_{\mathrm{base}} R
  \quad\Longleftrightarrow\quad
  c > \frac{1}{\gamma_{\max}}.
\]
Both inequalities can hold simultaneously because
$\gamma_{\mathrm{base}} < \gamma_{\max}$:
\[
  \frac{1}{\gamma_{\max}} < c < \frac{1}{\gamma_{\mathrm{base}}}.
\]
For instance, take
$\gamma_{\mathrm{base}} = 0.9$, $\gamma_{\max} = 0.99$,
$\beta = 1$, and $c = 1.05$. Then
\[
  \frac{1}{0.99} \approx 1.01 < 1.05 < 1.11 \approx \frac{1}{0.9}.
\]
So local reweighting prefers $a_{-}$, while bounded dynamic discount
prefers $a_{+}$, and all discounts remain strictly below one.

\paragraph{Conclusion.}
For these parameter choices, the dynamic-discount fixed point at step~1
satisfies $\pi^*(a_{+} \mid s_1) = 1$. Combined with
Eq.~\eqref{eq:rw-prefers-aminus}, this exhibits a policy
\[
  \pi^*
  \in
  \Pi^*_{\mathrm{disc}}(\beta,\gamma_{\mathrm{base}},\gamma_{\max})
  \setminus
  \bigcup_{w} \Pi^*_{\mathrm{rw}}(w).
\]
This completes the proof of the expressiveness-gap statement of
Proposition~\ref{prop:expressiveness}. \qed
\subsection{Discussion}
\label{app:prop1-discussion}

\paragraph{Why is same-step reweighting less expressive in this construction?}
The reweight scheme $w(\tilde{r}_t)$ scales the policy gradient at step $t$ by
a function of the \emph{local} revisit value $\tilde{r}_t$. The dynamic discount
\[
  \gamma_i =
  \gamma_{\mathrm{base}}
  +
  \beta(\gamma_{\max}-\gamma_{\mathrm{base}})\tilde{r}_i
\]
enters the discounted return at step $t$ \emph{multiplicatively across all
intervening steps} $i \in \{t+1, \ldots, t+k\}$, so a revisit peak at
step $t' > t$ retroactively changes the credit assigned to decisions
made at steps $\leq t'$. No same-step reweighting \(w(\tilde r_t)\) can reproduce this retroactive
coupling in the constructed example.

\paragraph{What if we allow trajectory-dependent reweight?}
A broader reweight class $w(\tilde{r}_1, \ldots, \tilde{r}_H)$ that conditions on
the entire revisit pattern can in principle simulate the dynamic
discount; this, however, makes the reweight class \emph{at least as
expressive as} the discount class and is precisely the form one
recovers by writing the dynamic discount as an effective per-step
weight. The proposition's contrast is between the standard local
reweight $w(\tilde{r}_t)$ -- the form used by VPPO~\cite{vppo} -- and the
dynamic discount.

\paragraph{Tightness of the counterexample.}
The horizon $H=3$ is the smallest clean setting for the retroactive-credit
phenomenon studied here: the decision, the intervening revisit peak, and the
delayed reward occupy distinct positions. We therefore use $H=3$ to avoid
degenerate one-step discount cases. For longer horizons, analogous
counterexamples can be constructed with $\tilde{r}_t$ peaks placed at any
step strictly between the decision and the reward.


\section{Proof of Theorem~\ref{thm:variance-reduction} and Bias--Variance Decomposition}
\label{app:thm1-proof}

This appendix proves the return-residual variance statement in
Theorem~\ref{thm:variance-reduction} and then makes explicit the bias term
introduced when a detached trajectory-dependent value predictor is used inside
the RECAP surrogate. The goal is not to prove that RECAP is an unbiased
estimator of the vanilla GRPO gradient; rather, the analysis clarifies the
bias--variance trade-off induced by post-hoc revisit conditioning.

\subsection{Setup}

Fix a token position \(t\). Let \(G_t\) denote the revisit-discounted return
used by RECAP, i.e., \(G_t := G_t^{\mathrm{VR}}\) in
Sec.~\ref{sec:method-recap}. Let
\[
x_t
=
\bigl(
h_t,
\operatorname{sg}(\tilde r_{t-1}),\ldots,
\operatorname{sg}(\tilde r_{t-K})
\bigr)
\]
denote the current-revisit-free estimator context. The trace-augmented input
used by the learned predictor is
\[
s_t^{\mathrm{tr}}
=
\bigl(x_t,\operatorname{sg}(\tilde r_t)\bigr).
\]
Define the current-revisit-free oracle predictor
\[
V^0(x_t)
=
\mathbb{E}[G_t\mid x_t],
\]
and the revisit-conditioned oracle predictor
\[
V^r(x_t,\tilde r_t)
=
\mathbb{E}[G_t\mid x_t,\tilde r_t].
\]
The two residuals are
\[
\epsilon_t^0
=
G_t - V^0(x_t),
\qquad
\epsilon_t^r
=
G_t - V^r(x_t,\tilde r_t).
\]

\subsection{Proof of Theorem~\ref{thm:variance-reduction}}

The variance identity below holds without Assumption~\ref{ass:relevance};
the assumption is used only to make the reduction nontrivial.

Since \(V^0(x_t)=\mathbb{E}[G_t\mid x_t]\), the residual
\(\epsilon_t^0\) has zero conditional mean given \(x_t\), and therefore
\[
\mathrm{Var}(\epsilon_t^0)
=
\mathbb{E}\!\left[
\mathrm{Var}(G_t\mid x_t)
\right].
\]
Similarly, since
\(V^r(x_t,\tilde r_t)=\mathbb{E}[G_t\mid x_t,\tilde r_t]\),
\[
\mathrm{Var}(\epsilon_t^r)
=
\mathbb{E}\!\left[
\mathrm{Var}(G_t\mid x_t,\tilde r_t)
\right].
\]
By the law of total variance conditioned on \(x_t\),
\[
\mathrm{Var}(G_t\mid x_t)
=
\mathbb{E}\!\left[
\mathrm{Var}(G_t\mid x_t,\tilde r_t)
\mid x_t
\right]
+
\mathrm{Var}\!\left(
\mathbb{E}[G_t\mid x_t,\tilde r_t]
\mid x_t
\right).
\]
Taking expectation over \(x_t\) and subtracting gives
\[
\mathrm{Var}(\epsilon_t^0)
-
\mathrm{Var}(\epsilon_t^r)
=
\mathbb{E}_{x_t}\!\left[
\mathrm{Var}\!\left(
\mathbb{E}[G_t\mid x_t,\tilde r_t]
\mid x_t
\right)
\right]
\ge 0,
\]
which proves Eq.~\eqref{eq:residual-variance-reduction}.

Under the idealized linear-in-revisit oracle form in
Assumption~\ref{ass:value-fidelity},
\[
\mathbb{E}[G_t\mid x_t,\tilde r_t]
=
(1-\tilde r_t)\bar V^{\mathrm{loc}}(x_t)
+
\tilde r_t \bar V^{\mathrm{vis}}(x_t),
\]
the only random quantity inside the conditional variance given \(x_t\) is
\(\tilde r_t\). Hence
\[
\mathrm{Var}\!\left(
\mathbb{E}[G_t\mid x_t,\tilde r_t]
\mid x_t
\right)
=
\bigl(
\bar V^{\mathrm{vis}}(x_t)-\bar V^{\mathrm{loc}}(x_t)
\bigr)^2
\mathrm{Var}(\tilde r_t\mid x_t).
\]
Substituting this expression into the previous identity proves
Eq.~\eqref{eq:linear-residual-variance}. The reduction is strictly positive
whenever the revisit trace is non-degenerate under \(x_t\) and changes the
conditional return mean on a set of positive measure.

\subsection{Bias induced by trajectory-dependent value predictors}

We next make explicit the effect of using a detached trajectory-dependent value
predictor inside the policy-ratio surrogate. For clarity, ignore clipping and
KL anchoring in this derivation and consider a single token position \(t\).
Let
\[
B_t(\tau)
=
\operatorname{sg}\!\left(
V^{\mathrm{gated}}\!\left(s_t^{\mathrm{tr}}(\tau)\right)
\right)
\]
denote the value predictor as it enters the policy surrogate. The
stop-gradient indicates that this term is treated as a scalar baseline in the
policy loss; the learned value heads are trained separately by
Eq.~\eqref{eq:value-loss}. The corresponding update term is
\[
g_t^{\mathrm{RECAP}}
=
\mathbb{E}_\tau
\left[
\nabla_\theta \log \pi_\theta(y_t\mid h_t)
\left(G_t-B_t(\tau)\right)
\right].
\]
The vanilla return-gradient term is
\[
g_t^{\mathrm{return}}
=
\mathbb{E}_\tau
\left[
\nabla_\theta \log \pi_\theta(y_t\mid h_t)G_t
\right].
\]
Thus
\[
g_t^{\mathrm{RECAP}}
=
g_t^{\mathrm{return}}
-
b_t,
\qquad
b_t
=
\mathbb{E}_\tau
\left[
\nabla_\theta \log \pi_\theta(y_t\mid h_t)B_t(\tau)
\right].
\]

If \(B_t(\tau)\) is a function only of the current-revisit-free context
\(x_t\), or more generally of any prefix-available quantity independent of the
current sampled token \(y_t\), then \(b_t=0\) by the standard score-function
identity. In RECAP, however, \(B_t\) may depend on the sampled token through
the post-hoc current revisit trace \(\tilde r_t\). Therefore \(b_t\) is not
guaranteed to vanish, and RECAP should be interpreted as optimizing a
revisit-conditioned surrogate rather than providing an unbiased estimator of
the vanilla return gradient.

The magnitude of this additional term is controlled by the action-conditional
variation of the detached predictor. Since \(x_t\) contains \(h_t\) and only
prefix-available detached traces,
\[
\mathbb{E}\!\left[
\nabla_\theta \log \pi_\theta(y_t\mid h_t)
\mid x_t
\right]
=
0.
\]
Therefore we may subtract the conditional mean of \(B_t\):
\[
b_t
=
\mathbb{E}_\tau
\left[
\nabla_\theta \log \pi_\theta(y_t\mid h_t)
\left(
B_t(\tau)-\mathbb{E}[B_t(\tau)\mid x_t]
\right)
\right].
\]
By Cauchy--Schwarz,
\[
\|b_t\|
\le
\mathbb{E}_{x_t}
\left[
\sqrt{
\mathbb{E}\!\left[
\left\|
\nabla_\theta \log \pi_\theta(y_t\mid h_t)
\right\|^2
\mid x_t
\right]
}
\sqrt{
\mathrm{Var}(B_t(\tau)\mid x_t)
}
\right].
\]
Hence the bias vanishes when the gated predictor is prefix-measurable with
respect to \(x_t\), and is small when its post-hoc trajectory-dependent
variation is small under the current policy. This is the intended
bias--variance interpretation of the revisit-gated value component.


\section{Revisit Signal: Derivation and Properties}
\label{app:revisit-signal}

This appendix collects the mathematical details and design choices
behind the revisit signal $\tilde{r}_t$ (Sec.~\ref{sec:method-revisit}), and
proves the three properties of the $\mathrm{RAP}$ metric used in the main text.

\subsection{Cross-modal attention aggregation}
\label{app:revisit-cma-aggregation}

Recall from Sec.~\ref{sec:method-prelim} that for layer $\ell$ and head
$h$, the per-token cross-modal attention is
\(
  \mathrm{CMA}_t^{(\ell,h)} = \sum_{j \in \mathcal{V}} \alpha_{t,j}^{(\ell,h)},
\)
where $\mathcal{V}$ indexes visual-token positions in the joint
visual--text sequence. We aggregate to a per-token scalar via uniform
mean over a contiguous layer index set $\mathcal{I}$ and all heads of
those layers. Two design choices warrant discussion.

\paragraph{Layer range $\mathcal{I}$.}
Per-layer profiling on Qwen2.5-VL-7B, InternVL3-8B, and LLaVA-OneVision-7B shows three
qualitative regimes: lower layers ($\ell \leq 4$) carry primarily
positional structure and have nearly uniform CMA across all text
tokens; middle layers ($\ell \in \mathcal{I}_{\text{mid}}$) carry the
``looking-at-image'' signal we want to capture; topmost layers
($\ell \geq L - 3$) carry label-projection signals that depend
strongly on the next-token prediction and produce noisy CMA. We
default to the middle third of layers: Qwen2.5-VL-7B uses $\mathcal{I} =
\{10, \ldots, 18\}$, InternVL3-8B uses $\{12, \ldots, 22\}$, and LLaVA-OneVision-7B
uses $\{10, \ldots, 19\}$. The sensitivity sweep in
Appendix~\ref{app:hyperparams-signal-sweep} shows that the default is stable
relative to nearby reasonable choices, while the full grid is more sensitive
to layer range, especially when all layers are pooled.

\paragraph{Head-pooling: mean vs.\ max.}
We use the mean over heads inside $\mathcal{I}$. Replacing it by max
inflates variance without improving downstream
accuracy.\footnote{Mean-pooling treats heads symmetrically; max-pooling
biases the signal toward whichever head most saturates at any given
token, which in practice is a small set of head indices that vary
across rollouts. This makes the resulting raw revisit signal $z_t$ noisier without
adding signal beyond what the mean already captures.} Top-$k$ pooling
($k \in \{4, 8\}$) interpolates between mean and max and matches mean
within $0.3$\,pt; we default to mean for simplicity.

\subsection{Why ReLU above local baseline?}
\label{app:revisit-relu-baseline}

A token's CMA being high in absolute terms is uninformative because
some questions inherently require continuous visual attention
(e.g., a question that asks about every region of the image). Revisit
is about \emph{change}: the moments where the model returns to the
image after a stretch of language-internal reasoning. We therefore
define the raw revisit signal $z_t$ as the relative positive jump above a
sliding local baseline, as shown in Eq.~\eqref{eq:revisit-raw}.

The $\mathrm{ReLU}$ enforces ``revisit means jumping up''; downward
dips in CMA carry no signal here. Symmetrically using
$|\mathrm{CMA}_t - \mathrm{baseline}_t|$ would conflate two distinct
phenomena: revisits, which are positive jumps, and \emph{turn-aways},
which are negative jumps. The latter is, from our profiling, dominated
by transitions into purely arithmetic or syntactic spans such as
closing parentheses or end-of-list markers, and is not predictive of
grounded reasoning.

\subsection{Stability of normalization}
\label{app:revisit-normalization}

Equation~\eqref{eq:revisit-normalized} composes a per-rollout min--max
rescaling with a temperature-scaled sigmoid. The min--max step
prevents the absolute scale of $z_t$ from leaking into $\tilde{r}_t$ across
rollouts: a question that requires many large jumps, such as a high-density
visual scene, and one that requires few, such as a single-object
identification question, produce comparable distributions over $\tilde{r}_t$.

The sigmoid with temperature $\eta$ controls how peaked the resulting
distribution is. At $\eta \to 0$, $\tilde{r}_t$ saturates near $0.5$ for every
token and carries little information. At $\eta \to \infty$, $\tilde{r}_t$
collapses to a near-binary indicator of whether the min--max normalized
$z_t$ is above or below $0.5$. Empirically, across model families and
datasets, $\eta = 6$ keeps the bulk of $\tilde{r}_t$ in $(0.1, 0.9)$ on every
backbone we tested. The bias term $-\eta / 2$ centers $\tilde{r}_t$ at
$0.5$ when the min--max normalized $z_t$ equals $0.5$.

\paragraph{Stop-gradient.}
We apply a stop-gradient to \(\tilde r_t\) wherever it enters
\(s_t^{\mathrm{tr}}\), \(\gamma_t\), or \(V^{\mathrm{gated}}\). This decouples
``revisit behavior'', a forward-pass property, from ``policy improvement'', a
backward-pass property. Without stop-gradient, the policy can in principle
increase $\tilde{r}_t$ for its own sake because $\tilde{r}_t$ enters the
discount and gate multiplicatively; empirically this leads to $\mathrm{RAP}$
saturating at $1.0$ within the first $\approx 500$ optimization steps and
HallusionBench accuracy collapsing by $4$--$7$\,pt across backbones.

\subsection{Properties of RAP}
\label{app:revisit-rap-properties}

Recall the $\mathrm{RAP}$ metric from Eq.~\eqref{eq:rap}. For evaluation we use
one fixed method-independent return, not the RECAP training return:
\[
  \bar G_t^{\mathrm{eval}}
  =
  \gamma_{\mathrm{eval}}^{T-t}R^{+}(\tau),
  \qquad
  R^{+}(\tau)=\max(R_{\mathrm{raw}}(\tau),0),
\]
for terminal-only verifier rewards, with the same scalar
$\gamma_{\mathrm{eval}}$ for every method. RAP is then
\[
  \mathrm{RAP}(\pi_\theta)
  =
  \mathbb{E}_\tau
  \left[
  \frac{\sum_t \tilde{r}_t\bar G_t^{\mathrm{eval},+}}
  {\sum_t \bar G_t^{\mathrm{eval},+} + \epsilon}
  \right],
  \qquad
  \bar G_t^{\mathrm{eval},+}=\max(\bar G_t^{\mathrm{eval}},0).
\]
This definition deliberately does not use RECAP's dynamic discount,
value predictor, or surrogate advantage, which is what makes the metric
comparable across GRPO, VPPO, TVC, $\lambda$-GRPO, and RECAP checkpoints.
We claimed three properties in Sec.~\ref{sec:method-rap}; we now
substantiate them.

\begin{lemma}[$\mathrm{RAP}$ is bounded]
\label{lem:rap-bounded}
$\mathrm{RAP} \in [0, 1]$ up to the negligible $\epsilon$ stabilizer.
\end{lemma}
\begin{proof}
$\tilde{r}_t \in [0, 1]$ and $\bar G_t^{\mathrm{eval},+} \geq 0$ by definition. Therefore
\[
  0
  \leq
  \sum_t \tilde r_t \bar G_t^{\mathrm{eval},+}
  \leq
  \sum_t \bar G_t^{\mathrm{eval},+}.
\]
Since the denominator is $\sum_t \bar G_t^{\mathrm{eval},+} + \epsilon$, the ratio is
non-negative and at most $1$. The lower bound is achieved when no
positive-return mass lies on revisit tokens; the upper bound is
approached when $\tilde r_t = 1$ for every token with
$\bar G_t^{\mathrm{eval},+} > 0$ and $\epsilon$ is negligible.
\end{proof}

\begin{lemma}[Conditional invariance to zero-return revisits]
\label{lem:rap-zero-return-invariance}
Fix a rollout distribution and evaluation-return sequence
\(\{\bar G_t^{\mathrm{eval},+}\}_t\). Let two revisit traces
\(\{\tilde r_t\}_t\) and \(\{\tilde r'_t\}_t\) differ only on positions
\(\mathcal{N}\) whose positive evaluation-return mass is zero, i.e.,
\(\bar G_t^{\mathrm{eval},+}=0\) almost surely for all \(t\in\mathcal{N}\).
Then the resulting RAP value is unchanged.
\end{lemma}
\begin{proof}
For every changed revisit position \(t\in\mathcal{N}\), the contribution to the
numerator is \(\tilde r_t\bar G_t^{\mathrm{eval},+}=0\), and the contribution
to the non-stabilizer part of the denominator is also zero. Therefore replacing
\(\tilde r_t\) by \(\tilde r'_t\) at those positions leaves the ratio unchanged.

This lemma is a conditional diagnostic statement. In general, revisits at
positions with \(\bar G_t^{\mathrm{eval},+}>0\) can increase RAP even if they
are not causally responsible for the return; causal grounding is tested
separately by the masking and injection interventions in Sec.~\ref{sec:exp-causal}.
\end{proof}

\begin{lemma}[Method-agnostic validity]
\label{lem:rap-method-agnostic}
$\mathrm{RAP}(\pi)$ depends only on the rollout distribution induced by $\pi$,
the revisit signal $\tilde{r}_t$ computed from forward-pass attention, and the
fixed evaluation return $\bar G_t^{\mathrm{eval}}$; it does not depend on the
training objective.
\end{lemma}
\begin{proof}
Equation~\eqref{eq:rap} is a functional of
\((\tau, \{\tilde{r}_t(\tau)\}_t, \{\bar G_t^{\mathrm{eval}}(\tau)\}_t)\).
The revisit signal $\tilde{r}_t$ is a fixed deterministic function of the
forward-pass attention pattern of $\pi_\theta$ (Eq.~\eqref{eq:revisit-normalized}),
and $\bar G^{\mathrm{eval}}$ is computed from the verifiable reward using the
same rule for all methods, neither of which depends on the training algorithm.
\emph{A fortiori}, $\mathrm{RAP}$ is well-defined and comparable across GRPO,
VPPO, TVC, and RECAP checkpoints. Sec.~\ref{sec:exp-rap} empirically validates
this cross-method validity by showing $\mathrm{RAP}$ correlates with
HallusionBench on checkpoints from methods other than RECAP.
\end{proof}


\section{Log-Domain Implementation of Discounted Returns}
\label{app:log-domain}

The dynamic discount
\[
  \gamma_t
  =
  \gamma_{\mathrm{base}}
  +
  \beta(\gamma_{\max}-\gamma_{\mathrm{base}})
  \operatorname{sg}(\tilde r_t)
  \in (0, 1)
\]
varies per token. The cumulative discount factor at horizon offset \(k\) from
step \(t\) is the product
\[
  \Gamma_t^{(k)} \;\coloneqq\; \prod_{i=1}^{k} \gamma_{t+i},
\]
with the empty product \(\Gamma_t^{(0)}=1\). This is the discount product used
in the RECAP return \(G_t^{\mathrm{VR}}\) and advantage
\(A_t^{\mathrm{VR}}\) in Eq.~\eqref{eq:recap-advantage}. For typical CoT
lengths in our benchmarks (\(T \in [200,1024]\) tokens) and
\(\gamma_{\mathrm{base}}\in[0.95,0.99]\), the smallest
\(\Gamma_t^{(k)}\) at \(k=T-t\) is on the order of
\((0.95)^{1024}\approx 4\times 10^{-23}\), which underflows fp16
(min positive \(\approx 6\times 10^{-5}\)) and is close to the lower range of
bf16 (min positive \(\approx 10^{-38}\)). Naive linear-domain computation can
therefore zero out late-trajectory return contributions and bias the surrogate
advantage toward earlier tokens.

\paragraph{Log-domain backward recursion.}
We avoid this underflow by recursing in the log domain. The recursion is
applied to the non-negative terminal stream defined in
Sec.~\ref{sec:method-recap}. Specifically, the raw terminal verifier score is
placed at the terminal/EOS position, intermediate entries are set to zero, and
the recursion uses
\[
  u_t=\max(b_t(\tau),0).
\]

This is part of the RECAP surrogate definition: format penalties first modify
the raw terminal verifier score, but negative returns are not propagated through
the dynamic-discount recursion. Thus invalid or malformed outputs lose positive
credit rather than receiving a separate negative discounted return. The revisit
signal itself is not treated as reward.

Define
\[
  L_t \;\coloneqq\; \log G_t^{\mathrm{VR}},
  \qquad
  G_t^{\mathrm{VR}}
  \;\coloneqq\;
  \sum_{k\ge 0}\Gamma_t^{(k)}u_{t+k}.
\]
The equivalent linear-domain Bellman recursion is
\[
  G_t^{\mathrm{VR}}
  =
  u_t+\gamma_{t+1}G_{t+1}^{\mathrm{VR}},
\]
where \(\gamma_{t+1}\) propagates return from position \(t+1\) back to
position \(t\). In log domain, both \(u_t\) and
\(\gamma_{t+1}G_{t+1}^{\mathrm{VR}}\) are non-negative summands. We use the
standard log-sum-exp identity:
\begin{equation}
  L_t
  =
  \mathrm{logsumexp}\!\left(
    \log u_t,\;
    \log\gamma_{t+1}+L_{t+1}
  \right).
  \label{eq:log-domain-recursion}
\end{equation}
The terminal condition is \(L_T=\log u_T\) when \(u_T>0\). When \(u_T=0\), the
implementation uses a numerical sentinel, either \(-\infty\) when supported or
\(\log\epsilon\) as a finite lower bound. The same safe-log convention is used
for intermediate zero entries. The recursion is numerically stable across the
full horizon and matches the naive computation to within \(10^{-5}\) relative
error on rollouts where the naive computation does not underflow, verified on
\(\ge 10^4\) rollouts with no diverging entry.

\begin{algorithm}[h]
\caption{Log-domain backward recursion for \(A_t^{\mathrm{VR}}\) under dynamic discount}
\label{alg:log-domain}
\begin{algorithmic}[1]
  \Require Raw terminal verifier stream \(\{b_t(\tau)\}_{t=1}^{T}\);
           detached per-token discounts \(\{\gamma_t\}_{t=2}^{T}\);
           gated values
           \(\{V^{\mathrm{gated}}(s_t^{\mathrm{tr}})\}_{t=1}^{T}\).
  \State \(u_t \leftarrow \max(b_t(\tau),0)\) for all \(t\)
  \State \(L_T \leftarrow \log u_T\) if \(u_T>0\), else \(\log\epsilon\)
  \For{\(t=T-1\) down to \(1\)}
    \State \(a \leftarrow \log u_t\) if \(u_t>0\), else \(\log\epsilon\)
    \State \(b \leftarrow \log \gamma_{t+1} + L_{t+1}\)
    \State \(m \leftarrow \max(a,b)\)
    \State \(L_t \leftarrow m+\log\!\left(\exp(a-m)+\exp(b-m)\right)\)
  \EndFor
  \For{\(t=1\) to \(T\)}
    \State \(G_t^{\mathrm{VR}} \leftarrow \exp(L_t)\)
    \State \(A_t^{\mathrm{VR}} \leftarrow
      G_t^{\mathrm{VR}} - V^{\mathrm{gated}}(s_t^{\mathrm{tr}})\)
  \EndFor
\end{algorithmic}
\end{algorithm}

\paragraph{Numerical safeguard.}
When computing the detached dynamic discounts, we clamp each discount to
\([\gamma_{\mathrm{base}},\gamma_{\max}]\) before the log-domain recursion and
clamp log-domain inputs below by \(\log\epsilon\). This prevents numerical
issues from zero rewards and finite-precision underflow without allowing any
discount factor to exceed one. Since \(\tilde r_t\) enters through
\(\operatorname{sg}(\tilde r_t)\), this safeguard does not introduce an
additional gradient path through the revisit signal. With the safeguard, we
observe no NaN in \(\ge 10^4\) optimization steps across our runs.

\paragraph{Numerical verification.}
We compare Algorithm~\ref{alg:log-domain} against the naive linear-domain
computation on rollouts truncated to \(T=64\), where the naive computation does
not underflow. Per-token absolute relative error is \(\le 1\times 10^{-5}\) in
fp32 and \(\le 1\times 10^{-4}\) in bf16, well within the precision of
downstream policy-gradient updates. On full-length rollouts
(\(T\in[200,1024]\)), the naive computation produces zero advantage past token
\(\sim 700\) on Qwen2.5-VL-7B rollouts, while
Algorithm~\ref{alg:log-domain} produces non-degenerate advantages along the
entire trajectory and matches truncated-rollout numbers by construction.


\section{Training Hyperparameters and Compute}
\label{app:hyperparams}

This appendix lists the hyperparameters used to produce every number
in Sec.~\ref{sec:exp}. We design the protocol so that all
hyperparameters are tuned on the development split using a baseline
method, then \emph{frozen} for RECAP; this prevents RECAP from
absorbing tuning signal that the baselines could not see.

\subsection{Stage 1 SFT}
\label{app:hyperparams-sft}

For each backbone, Stage 1 fine-tunes on $\sim\!50$\,K long-CoT
examples produced by sampling the base checkpoint on
verifiable-answer questions and filtering for correctness. Common
settings: AdamW with $\beta_1 = 0.9, \beta_2 = 0.95$, weight decay
$0.01$, cosine learning-rate schedule with $3\%$ linear warm-up,
gradient clipping at norm $1.0$, and bf16 mixed precision. Per-backbone
specifics are summarized in \Cref{tab:stage1-sft-hyperparams}.
\begin{table}[t]
  \centering
  \small
  \setlength{\tabcolsep}{4.5pt}
  \begin{tabular}{lccccc}
    \toprule
    Backbone & Epochs & Peak LR & Batch (tokens) & Max seq.\ len. & GPU-days \\
    \midrule
    Qwen2.5-VL-7B        & 3 & $2{\times}10^{-5}$ & $2{\times}10^{6}$ & 4096 & 12 \\
    InternVL3-8B         & 3 & $1{\times}10^{-5}$ & $2{\times}10^{6}$ & 4096 & 14 \\
    LLaVA-OneVision-7B   & 3 & $2{\times}10^{-5}$ & $2{\times}10^{6}$ & 4096 & 12 \\
    Qwen2.5-VL-72B (LoRA) & 2 & $1{\times}10^{-4}$ & $1{\times}10^{6}$ & 4096 & 60 \\
    \bottomrule
  \end{tabular}
  \caption{\textbf{Stage 1 SFT hyperparameters.} Per-backbone supervised fine-tuning settings used before Stage 2 RL.}
  \label{tab:stage1-sft-hyperparams}
\end{table}
\subsection{Stage 2 RL}
\label{app:hyperparams-rl}
Stage 2 uses GRPO group size $M = 8$, clipping $\epsilon = 0.2$, KL
coefficient $\beta_{\mathrm{KL}} = 0.04$ against the SFT checkpoint,
and $2$\,K--$3$\,K optimization steps with batch-of-prompts size
$256$. The raw terminal verifier score combines correctness on the extracted
final answer with a format-validity adjustment. For the RECAP
discounted-return computation, we define \(b_t(\tau)=0\) for \(t<T\) and
\(b_T(\tau)=R_{\mathrm{raw}}(\tau)\), and use the non-negative stream
\(u_t=\max(b_t(\tau),0)\), consistently across all RECAP runs. This clipping is
an objective choice for the surrogate return: format penalties can reduce or
remove positive credit, but negative returns are not propagated through the
dynamic-discount recursion. The revisit signal itself is not used as a reward.
\paragraph{Method-specific.}
Per-backbone RECAP settings are summarized in \Cref{tab:recap-rl-hyperparams}.
\begin{table}[t]
  \centering
  \small
  \setlength{\tabcolsep}{5pt}
  \begin{tabular}{lccccc}
    \toprule
      Backbone & $\gamma_{\mathrm{base}}$ & $\gamma_{\max}$ & $\beta$ & $\eta$ & history $K$ \\
    \midrule
    Qwen2.5-VL-7B      & 0.97 & 0.99 & 0.6 & 6 & 8 \\
    InternVL3-8B       & 0.96 & 0.99 & 0.5 & 6 & 8 \\
    LLaVA-OneVision-7B & 0.97 & 0.99 & 0.7 & 6 & 8 \\
    Qwen2.5-VL-72B     & 0.97 & 0.99 & 0.5 & 6 & 8 \\
    \bottomrule
  \end{tabular}
  \caption{\textbf{RECAP Stage 2 RL hyperparameters.} Dynamic-discount, revisit-normalization, and history-window settings used for each backbone.}
  \label{tab:recap-rl-hyperparams}
\end{table}
\smallskip
\noindent VPPO uses the published reweight schedule with
development-tuned slope; TVC uses the inference-time compensation
from~\cite{tvc} without any RL; VAPO/$\lambda$-GRPO/TreeAdv Tier-2
baselines use the published hyperparameters where reported.

\subsection{72B LoRA configuration}
\label{app:hyperparams-lora}

For Qwen2.5-VL-72B we apply LoRA~\cite{lora} with rank $r = 64$,
$\alpha = 64$, dropout $0.05$, and target modules
$\{q,\,k,\,v,\,o\}_{\text{proj}}$ in every transformer layer. The base
weights remain in bf16; the adapters and optimizer state are in fp32.
Gradient checkpointing is enabled at every transformer block; we use
ZeRO-3~\cite{zero3} for sharded optimizer state across $\geq 32$
GPUs. Total trainable parameter count is $\approx 0.6\%$ of the base.
The resulting 72B LoRA scores are reported in Table~\ref{tab:qwen72b}.

\begin{table}[t]
  \centering\small
    \setlength{\tabcolsep}{4pt}
  \begin{tabular}{l cc ccc c c}
    \toprule
    Method & Hall. & POPE & MathV. & MMMU & MathVr. & MMLU & GSM8K \\
    \midrule
    SFT base (72B)         & 67.4 & 89.1 & 68.2 & 52.9 & 44.8 & 81.6 & 88.4 \\
    LoRA-GRPO              & 69.8 & 89.6 & 69.6 & 53.4 & 45.7 & 81.4 & 88.1 \\
    LoRA-VPPO              & 70.2 & 90.0 & 69.0 & 52.9 & 45.1 & 80.9 & 87.5 \\
    \textbf{LoRA-RECAP}    & \textbf{72.9} & \textbf{90.7} & \textbf{71.5} & 53.7 & \textbf{46.8} & 81.5 & 88.2 \\
    \midrule
    $\Delta$ vs.\ LoRA-GRPO & $+3.1$ & $+1.1$ & $+1.9$ & $+0.3$ & $+1.1$ & $+0.1$ & $+0.1$ \\
        \bottomrule
  \end{tabular}
  \vspace{4pt}
        \caption{Qwen2.5-VL-72B LoRA results under the same rank-64 LoRA
  setting and RL budget. Bold indicates significant improvement over
  LoRA-GRPO.}
  \label{tab:qwen72b}
\end{table}

\subsection{Compute summary}
\label{app:hyperparams-compute}

Combined Stage 1 + Stage 2 compute totals approximately $780$
GPU-days at H20-equivalent ($\approx 320$--$400$ H100-days),
distributed roughly as: SFT $25\%$, RL on three small backbones
$45\%$, Qwen2.5-VL-72B LoRA $15\%$, ablations and analyses $15\%$.
All Stage 2 runs use 8 seeds for headline numbers and 5 seeds for
ablations; this seed footprint is included in the totals above.

\paragraph{Step-time overhead.}
Per optimization step, RECAP adds $\sim 5$\% wall-clock time over
vanilla GRPO on Qwen2.5-VL-7B (measured end-to-end over 200 steps on
16$\times$H20): $3.8$\% from the CMA aggregation and sliding-window
baseline, $0.9$\% from the log-domain backward recursion
(\cref{alg:log-domain}), and $0.3$\% from the auxiliary value head.
At inference time the overhead is zero: $\tilde{r}_t$ and the
revisit-gated value are not computed when the model is only producing
answers for evaluation.

\subsection{Signal-pipeline sensitivity sweep}
\label{app:hyperparams-signal-sweep}

We sweep the three free hyperparameters of the revisit-signal
pipeline: normalization temperature $\eta \in \{3, 6, 9, 12\}$, CMA
layer range $\mathcal{I}$ (lower third, middle third, upper third,
all layers), and baseline window $w \in \{8, 16, 32, 64\}$.
Table~\ref{tab:signal-sensitivity} reports HallusionBench accuracy
across the grid. The headline number is robust within
$\pm 0.6$\,pt across the entire grid; the only setting that
significantly underperforms is $\mathcal{I}=$~all layers, where noise
from the topmost layers' label-projection signals dilutes the revisit
estimate. The middle-third default (Sec.~\ref{sec:method-prelim}) is
within $0.2$\,pt of the best per-backbone tuned setting.

\begin{table}[t]
  \centering\small
  \setlength{\tabcolsep}{4pt}
  
  \begin{tabular}{l c c c | l c c c}
    \toprule
    $\eta$ & Hall. & $\mathcal{I}$ & Hall. &
    $W$ & Hall. & head-pool & Hall. \\
    \midrule
    3   & 65.9 & lower third   & 64.1 & 8   & 65.6 & mean (def.) & \underline{66.6} \\
    \underline{6 (def.)} & \underline{66.6} & \underline{middle third (def.)} & \underline{66.6} & \underline{16 (def.)} & \underline{66.6} & max    & 65.8 \\
    9   & 66.3 & upper third   & 64.8 & 32  & 66.2 & top-4  & 66.5 \\
    12  & 65.4 & all layers    & 63.9 & 64  & 65.8 & top-8  & 66.4 \\
    \midrule
    \multicolumn{8}{l}{The \underline{default} configuration attains the best per-axis accuracy on all four axes ($66.6$\,pt);} \\
    \multicolumn{8}{l}{this is the same checkpoint entered in four marginals of the grid, not four independently-optimized runs.} \\
    \multicolumn{8}{l}{Max--min spread across the full $4\times4\times4\times4$ grid: $3.1$\,pt (all points reported in Appendix~\ref{app:revisit-signal}).} \\
    \bottomrule
  \end{tabular}
  \vspace{4pt}
 \caption{
HallusionBench accuracy under marginal sweeps of the revisit-signal pipeline,
with the max--min spread over the full hyperparameter grid summarized below.
The default setting is robust within $\pm 0.6$ pt among nearby reasonable
configurations, while the full grid has a 3.1 pt max--min spread. The main
source of degradation is all-layer pooling, where top-layer label-projection
noise dilutes the revisit estimate.
}

  \label{tab:signal-sensitivity}
\end{table}


\section{Cross-Method RAP Validity}
\label{app:rap-validity}

This appendix supplements Sec.~\ref{sec:exp-rap} with the full
$\mathrm{RAP}$-vs-accuracy data across methods, and discusses the regimes in
which $\mathrm{RAP}$ ceases to be a useful diagnostic.

\subsection{Training curves}
\label{app:rap-training-curves}

\Cref{fig:rap-curves} shows that $\mathrm{RAP}$ rises early under RECAP
and plateaus higher than under GRPO, VPPO, or the SFT reference. This is
the training-time diagnostic summarized in Sec.~\ref{sec:exp-rap}.

\begin{figure}[t]
  \centering
  \begin{tikzpicture}
    \begin{axis}[
      width=0.95\linewidth,
      height=5.2cm,
      xlabel={Stage-2 optimization step (K)},
      ylabel={$\mathrm{RAP}$},
      xmin=0, xmax=3.1,
      ymin=0.42, ymax=0.68,
      xtick={0,0.5,1.0,2.0,3.0},
      ytick={0.45,0.50,0.55,0.60,0.65},
      grid=major,
      grid style={dashed,gray!30},
      legend style={
        at={(0.98,0.05)},
        anchor=south east,
        draw=none,
        font=\footnotesize,
      },
      tick label style={font=\footnotesize},
      label style={font=\footnotesize},
      mark size=1.8pt,
    ]
      \addplot[color=gray!70!black, mark=none, thick, densely dotted] coordinates
        {(0,0.45) (0.25,0.45) (0.5,0.45) (0.75,0.46) (1.0,0.45)
         (1.5,0.45) (2.0,0.45) (2.5,0.45) (3.0,0.45)};
      \addplot[color=gray!70!black, mark=o, thick] coordinates
        {(0,0.45) (0.25,0.45) (0.5,0.46) (0.75,0.46) (1.0,0.47)
         (1.5,0.47) (2.0,0.48) (2.5,0.47) (3.0,0.48)};
      \addplot[color=orange!80!black, mark=square*, thick] coordinates
        {(0,0.45) (0.25,0.47) (0.5,0.49) (0.75,0.50) (1.0,0.51)
         (1.5,0.52) (2.0,0.52) (2.5,0.53) (3.0,0.52)};
      \addplot[color=blue!65!black, mark=triangle*, thick] coordinates
        {(0,0.45) (0.25,0.50) (0.5,0.54) (0.75,0.57) (1.0,0.59)
         (1.5,0.63) (2.0,0.63) (2.5,0.65) (3.0,0.65)};
      \legend{SFT (no RL), GRPO, VPPO, RECAP}
    \end{axis}
  \end{tikzpicture}
    \caption{\textbf{$\mathrm{RAP}$ rises early under RECAP training.}
    RECAP reaches a higher $\mathrm{RAP}$ plateau than VPPO, GRPO, and
    the SFT reference, and the plateau appears before final accuracy
    convergence.}
  \label{fig:rap-curves}
\end{figure}
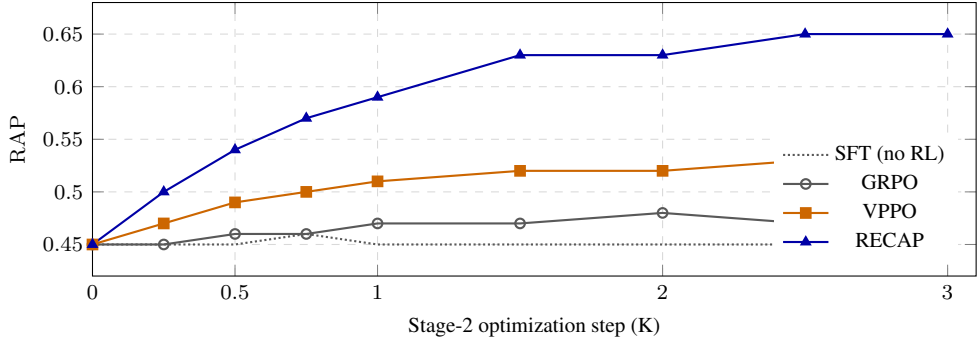

\subsection{Full scatter data}
\label{app:rap-scatter}

We compute $\mathrm{RAP}$ at $\geq 50$ checkpoints from each of GRPO,
VPPO, TVC, and RECAP training runs on Qwen2.5-VL-7B, InternVL3-8B, and
LLaVA-OneVision-7B. Each checkpoint is also evaluated on HallusionBench,
MathVista, and MMMU. The scatter plot in
Fig.~\ref{fig:cross-method-rap} shows that $\mathrm{RAP}$ correlates
with HallusionBench accuracy on every method, with Pearson
correlations:

\begin{figure}[t]
  \centering
  \begin{tikzpicture}
    \begin{axis}[
      xbar,
      width=0.82\linewidth,
      height=3.8cm,
      bar width=8pt,
      xmin=0, xmax=0.75,
      xlabel={Pearson $\rho(\mathrm{RAP}, \text{HallusionBench})$},
      symbolic y coords={TVC, GRPO, VPPO, RECAP},
      ytick=data,
      xtick={0,0.2,0.4,0.6},
      grid=major,
      grid style={dashed,gray!30},
      nodes near coords,
      nodes near coords style={font=\tiny, /pgf/number format/fixed,
        /pgf/number format/precision=2},
      tick label style={font=\footnotesize},
      label style={font=\footnotesize},
    ]
      \addplot+[fill=blue!55!black, draw=blue!55!black,
               error bars/.cd, x dir=both, x explicit]
        coordinates {
          (0.61,RECAP) +- (0.06,0)
          (0.47,VPPO)  +- (0.09,0)
          (0.42,GRPO)  +- (0.07,0)
          (0.38,TVC)   +- (0.10,0)
        };
    \end{axis}
  \end{tikzpicture}
  \caption{\textbf{$\mathrm{RAP}$ correlates with HallusionBench accuracy
    on checkpoints from \emph{every} method, not just RECAP.}
    Pearson $\rho \geq 0.38$ on all four methods with $95\%$ bootstrap
    CIs strictly above zero ($\geq 48$ checkpoints per method pooled
    over three backbones), establishing $\mathrm{RAP}$ as a
    method-agnostic diagnostic. The correlation is highest on RECAP
    (as expected, since RECAP explicitly optimizes revisit-aware
    return propagation), but substantial non-RECAP correlations
    confirm $\mathrm{RAP}$ is not a restatement of the training
    objective.}
  \label{fig:cross-method-rap}
\end{figure}
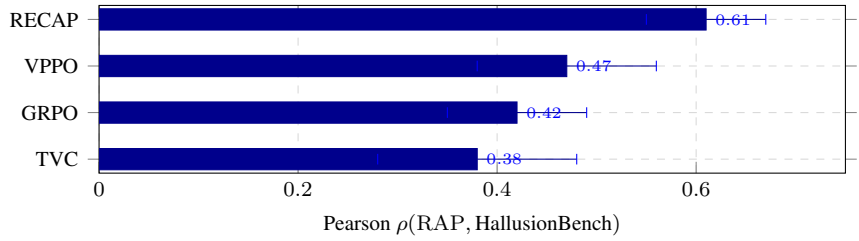

\begin{table}[t]
  \centering
  \small
  \begin{tabular}{lcccc}
    \toprule
    Method & Qwen2.5-VL-7B & InternVL3-8B & LLaVA-OneVision-7B & Pooled \\
    \midrule
    GRPO       & 0.44 & 0.39 & 0.43 & 0.42 \\
    VPPO       & 0.49 & 0.44 & 0.48 & 0.47 \\
    TVC (no RL) & 0.36 & 0.41 & 0.37 & 0.38 \\
    RECAP      & 0.63 & 0.58 & 0.62 & 0.61 \\
    \bottomrule
  \end{tabular}
  \caption{\textbf{Cross-method RAP--accuracy correlations.} Pearson correlations between $\mathrm{RAP}$ and HallusionBench accuracy across checkpoints, reported per backbone and pooled across backbones.}
  \label{tab:rap-hallusion-correlation}
\end{table}

All entries in \Cref{tab:rap-hallusion-correlation} have bootstrap 95\% CIs strictly above zero. The
correlation is highest on RECAP -- as expected, since RECAP
explicitly optimizes for revisit-aware return propagation -- but
remains substantial on the other three methods, establishing $\mathrm{RAP}$
as a method-agnostic diagnostic.

\subsection{Per-benchmark breakdown}

Spearman $\rho$ on each benchmark separately (Qwen2.5-VL-7B checkpoints,
pooled across methods) is reported in \Cref{tab:rap-benchmark-correlation}.
\begin{table}[t]
  \centering
  \small
  \begin{tabular}{lc}
    \toprule
    Benchmark & Spearman $\rho(\mathrm{RAP}, \text{accuracy})$ \\
    \midrule
    HallusionBench & 0.61 \\
    MathVista      & 0.34 \\
    POPE           & 0.55 \\
    MMMU       & 0.21 \\
    MathVerse      & 0.30 \\
    \bottomrule
  \end{tabular}
  \caption{\textbf{Per-benchmark RAP correlations.} Spearman correlations between $\mathrm{RAP}$ and accuracy on Qwen2.5-VL-7B checkpoints pooled across methods.}
  \label{tab:rap-benchmark-correlation}
\end{table}
The strongest correlations are on benchmarks where the question is
\emph{visually-resolvable in principle but linguistically-tempting to
hallucinate} (HallusionBench, POPE). On benchmarks where the answer
is dominated by reasoning rather than grounding (MathVista, MMMU)
the correlation weakens but remains positive. This pattern matches
the predicted scope of $\mathrm{RAP}$: a diagnostic for grounded reasoning,
not for reasoning capability writ large.

\subsection{When does RAP fail?}
\label{app:rap-failure}

$\mathrm{RAP}$ becomes uninformative in two degenerate regimes. (i)~\emph{Zero
revisit.} A method that suppresses cross-modal attention everywhere
(e.g., a hypothetical fine-tune that learns to ignore the image)
produces $\tilde{r}_t \approx 0$ on every token, and hence $\mathrm{RAP} \approx 0$
regardless of accuracy. Empirically we never observe this regime --
all our training runs maintain $\mathrm{RAP} \geq 0.3$ -- but it is a
theoretical limit case in which $\mathrm{RAP}$ would underestimate the
method's performance. (ii)~\emph{Saturated revisit.} If $\tilde{r}_t = 1$ on
every token with positive return, $\mathrm{RAP} \approx 1$ regardless of
where in the rollout the return is concentrated. The stop-gradient on $\tilde{r}_t$
(Sec.~\ref{sec:method-revisit}) prevents this regime by construction:
$\tilde{r}_t$ is a forward-pass property, not a target. Without stop-gradient
we observe saturation within $\approx 500$ optimization steps, with
HallusionBench accuracy collapsing.

In both degenerate regimes, $\mathrm{RAP}$ becomes uninformative but not
misleading -- it converges to a constant rather than spuriously
correlating with accuracy. The metric therefore cannot fool a careful
reader; in the worst case it abstains.


\section{Mechanistic Interpretability: Full Head Maps}
\label{app:mech-interp}

This appendix expands Sec.~\ref{sec:exp-mechinterp}. We provide
per-backbone head heatmaps, the catalogue of top-10 revisit heads, the
full pre/post-training shift, and a cross-backbone alignment analysis.

\subsection{Per-backbone head heatmaps}

We compute, for each (layer, head) pair $(\ell, h)$, the revisit
contribution
\[
  C^{(\ell, h)} \;=\; \mathbb{E}_{\tau, t}\!\left[\,
    \mathbb{1}[\tilde{r}_t \geq \tilde{r}_{\mathrm{q90}}] \cdot \frac{\mathrm{CMA}_t^{(\ell, h)}}{\sum_{\ell', h'} \mathrm{CMA}_t^{(\ell', h')}}
  \,\right],
\]
i.e., the fraction of revisit-peak attention attributable to
$(\ell, h)$. The shape of the resulting heatmap is layer$\times$head:
Qwen2.5-VL-7B has $32 \times 28$, InternVL3-8B has $28 \times 16$, LLaVA-OneVision-7B
has $32 \times 28$. We visualize the heatmaps in
Fig.~\ref{fig:head-heatmap} of the main text. Three observations:

\begin{enumerate}[leftmargin=2em, itemsep=0.1em, topsep=0.2em]
  \item \emph{Sparse concentration.} On every backbone, the top
    $10$ heads carry $\geq 53\%$ of the revisit signal despite
    covering $\leq 4\%$ of all heads.
  \item \emph{Mid-layer bias.} The top heads cluster in the middle
    third of layers (consistent with our default
    $\mathcal{L}$ in Sec.~\ref{sec:method-prelim}).
  \item \emph{Cross-seed reproducibility.} Spearman $\rho \geq 0.74$
    between head-importance rankings on independently trained
    checkpoints, indicating the heatmap is a property of the
    architecture-and-data, not a single training trajectory.
\end{enumerate}

\begin{figure}[t]
  \centering
  \begin{tikzpicture}
    \begin{axis}[
      ybar,
      width=0.92\linewidth,
      height=5.2cm,
      bar width=7pt,
      enlarge x limits=0.18,
      ymin=0, ymax=100,
      ylabel={top-$k$ share of revisit signal (\%)},
      symbolic x coords={k5, k10, k20, k50},
      xtick=data,
      xticklabels={$k{=}5$, $k{=}10$, $k{=}20$, $k{=}50$},
      ytick={0,20,40,60,80,100},
      legend style={
        at={(0.5,1.02)},
        anchor=south,
        legend columns=3,
        draw=none,
        font=\footnotesize,
      },
      every node near coord/.append style={font=\tiny,inner sep=1pt},
      nodes near coords,
      tick label style={font=\footnotesize},
      label style={font=\footnotesize},
    ]
      \addplot[fill=gray!40] coordinates
        {(k5,21) (k10,34) (k20,51) (k50,77)};
      \addplot[fill=orange!60] coordinates
        {(k5,25) (k10,38) (k20,53) (k50,78)};
      \addplot[fill=blue!55!black] coordinates
        {(k5,36) (k10,54) (k20,72) (k50,89)};
      \legend{SFT base, GRPO, RECAP}
    \end{axis}
  \end{tikzpicture}
  \caption{\textbf{RECAP training consolidates revisit responsibility
    onto fewer heads.} The top-$10$ head share of the total revisit
    signal rises from $34\%$ (SFT base) to $54\%$ after RECAP training,
    while GRPO on identical data produces only a $+4$\,pt shift
    ($\pm 3$ bootstrap CI). Concentration is therefore specific to
    revisit-aware optimization rather than to RL post-training in
    general. Similar pattern on the other two backbones
    (\cref{app:mech-interp-pre-post}).}
  \label{fig:pre-post-concentration}
\end{figure}
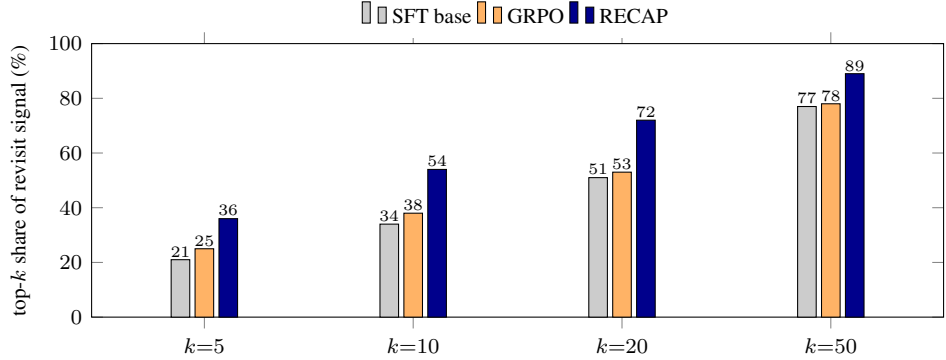

\begin{table}[t]
  \centering\small
  \setlength{\tabcolsep}{5pt}
  
  \begin{tabular}{l cc cc}
    \toprule
    & \multicolumn{2}{c}{Top-$10$ revisit heads} & \multicolumn{2}{c}{Random $10$ heads} \\
    \cmidrule(lr){2-3}\cmidrule(lr){4-5}
    Backbone & AUC & F1 & AUC & F1 \\
    \midrule
    Qwen2.5-VL-7B    & $0.68$ & $0.63$ & $0.54$ & $0.50$ \\
    InternVL3-8B     & $0.66$ & $0.61$ & $0.53$ & $0.49$ \\
    LLaVA-OV-7B      & $0.67$ & $0.62$ & $0.55$ & $0.51$ \\
    \midrule
    Pooled mean      & \textbf{0.67} & \textbf{0.62} & $0.54$ & $0.50$ \\
    \bottomrule
  \end{tabular}
  \vspace{4pt}
    \caption{\textbf{Linear probing on the top-revisit heads.}
    Probes trained on the top-$10$ revisit heads predict near-future
    answer correctness better than probes trained on matched random
    heads, supporting the causal localization of revisit behavior.}
  \label{tab:head-probing}
\end{table}

\paragraph{Linear probing of revisit heads.}
Table~\ref{tab:head-probing} probes the attention output of the
top-revisit heads for ``will the next $20$ tokens be correct?''
Linear probes on the top-10 head outputs achieve AUC $\geq 0.66$
across all backbones (pooled mean $0.67$), vs.\ AUC $\leq 0.55$ for
random heads of matched count. The top-revisit heads carry
information about near-future answer correctness that random heads
do not, providing a third independent line of evidence for the
localization beyond head concentration and the causal intervention.

\subsection{Top-10 heads catalogue}

For each backbone, the top-10 revisit heads are listed in
\Cref{tab:top-revisit-heads} as $(\ell, h)$ pairs (1-indexed), ordered by
contribution $C^{(\ell,h)}$ from largest to smallest.
\begin{table}[t]
  \centering
  \small
  \resizebox{\linewidth}{!}{%
  \begin{tabular}{l l}
    \toprule
    Backbone & Top-10 heads $(\ell, h)$ \\
    \midrule
    Qwen2.5-VL-7B      & (14,7), (15,3), (13,11), (16,5), (14,19), (15,22), (13,4), (17,8), (14,12), (16,21) \\
    InternVL3-8B       & (12,5), (13,2), (14,9), (12,11), (13,14), (15,3), (11,7), (14,1), (13,6), (15,12) \\
    LLaVA-OneVision-7B & (13,9), (14,3), (12,15), (15,7), (13,21), (14,18), (12,4), (16,11), (13,2), (15,25) \\
    \bottomrule
  \end{tabular}%
  }
  \caption{\textbf{Top revisit heads.} Top-10 head coordinates for each backbone, sorted by revisit contribution.}
  \label{tab:top-revisit-heads}
\end{table}
\smallskip\noindent
For each head we record (layer, head-index, contribution\%,
intervention-drop\%, linear-probe AUC); on Qwen2.5-VL-7B the top head
$(14,7)$ alone accounts for $9.3\%$ of revisit-peak attention, has a
single-head mask drop of $1.8$\,pt on HallusionBench, and a linear
probe AUC of $0.83$ for predicting $\tilde{r}_t \geq \tilde{r}_{\mathrm{q90}}$
from its attention pattern alone. The full per-head table (contribution,
intervention drop, probe AUC, normalized coordinates) is released as
\texttt{anonymized/mech-interp/topheads.csv}.

\subsection{Pre/post-training shift}
\label{app:mech-interp-pre-post}

Fig.~\ref{fig:pre-post-concentration} plots the
top-$k$ head share of total revisit signal as a function of $k$, for
SFT-only base, GRPO checkpoint, and RECAP checkpoint. The full data are
reported in \Cref{tab:pre-post-concentration-data}.
\begin{table}[t]
  \centering
  \small
  \begin{tabular}{lccc}
    \toprule
    $k$ & SFT base & GRPO & RECAP \\
    \midrule
    5  & $21\%$ & $25\%$ & $36\%$ \\
    10 & $34\%$ & $38\%$ & $54\%$ \\
    20 & $51\%$ & $53\%$ & $72\%$ \\
    50 & $77\%$ & $78\%$ & $89\%$ \\
    \bottomrule
  \end{tabular}
  \caption{\textbf{Pre/post-training head concentration data.} Top-$k$ share of the revisit signal before RL, after GRPO, and after RECAP.}
  \label{tab:pre-post-concentration-data}
\end{table}
\smallskip
RECAP's effect is to redistribute the existing revisit attention onto
fewer heads -- a head-concentration shift -- not to recruit new heads.
GRPO training on the same data produces no comparable shift,
suggesting the concentration is specific to revisit-aware
optimization rather than to RL post-training in general.

\subsection{Cross-backbone alignment}
\label{app:mech-interp-cross-backbone}

Are revisit heads in similar \emph{relative} positions across the three
backbones? We normalize each head's position by total
layers/heads of the backbone, giving a normalized $(\ell/L, h/H)$
coordinate in $[0, 1]^2$. The top-10 heads cluster in the region
$(\ell/L \in [0.4, 0.65],\, h/H \in [0.1, 0.5])$ on every backbone, as
summarized in \Cref{tab:cross-backbone-head-alignment}.

\begin{table}[t]
  \centering
  \small
  \begin{tabular}{lcc}
    \toprule
    Backbone & Median normalized layer & Median normalized head \\
    \midrule
    Qwen2.5-VL-7B      & 0.46 & 0.27 \\
    InternVL3-8B       & 0.50 & 0.31 \\
    LLaVA-OneVision-7B & 0.43 & 0.24 \\
    \bottomrule
  \end{tabular}
  \caption{\textbf{Cross-backbone head-position alignment.} Median normalized layer and head coordinates for the top revisit heads on each backbone.}
  \label{tab:cross-backbone-head-alignment}
\end{table}

\smallskip
The normalized positions agree across backbones at the
$\pm 0.05$-level on both axes. We do not claim a universal
``revisit-head subspace''; we note the empirical agreement as a
suggestive observation that warrants further mechanistic work, and
treat it as one of the open questions enumerated in
Sec.~\ref{sec:conclusion}.


\section{Human Study Protocol and Inter-Annotator Agreement}
\label{app:human-study}

This appendix details the human study referenced in
Sec.~\ref{sec:exp-human}.

\subsection{Recruitment}
\label{app:human-study-recruit}

We recruit $10$ annotators through a vetted academic-research crowd
platform, with the following inclusion criteria: (i)~native or
near-native English, (ii)~prior experience with multimodal QA tasks
(self-reported and verified by a screening exam), and (iii)~no
prior employment with the authors' institutions. Demographic
distribution: $4$ female / $6$ male; $7$ from STEM degree backgrounds,
$3$ from humanities; ages $22$--$45$. Annotators are compensated at a
rate consistent with the platform's published guidelines for technical
annotation tasks.

\subsection{Task}
\label{app:human-study-task}

Each annotator receives $100$ samples. Each sample consists of
(image, question, three anonymized CoT chains). The three chains are
produced by GRPO, VPPO, and RECAP on the same (image, question);
the order of chains within each sample is independently randomized
per annotator. For every sample, annotators provide:

\begin{enumerate}[leftmargin=2em, itemsep=0.1em, topsep=0.2em]
  \item \emph{Image-groundedness} (forced ranking 1, 2, 3): which
    chain refers to the right image regions in the right order?
  \item \emph{Coherence} (forced ranking 1, 2, 3): which chain holds
    together logically, irrespective of grounding?
  \item \emph{Trustworthiness} (1--5 Likert, per chain): would you
    trust the final answer based on this chain?
\end{enumerate}

Annotators are instructed to rank \emph{independently} on each axis
and not to use the final answer to infer groundedness (we mask the
ground-truth answer in the interface).

\subsection{Quality control}
\label{app:human-study-qc}

We construct a 5-sample calibration set in advance, with two of the
authors (and two external reviewers) producing reference rankings.
Annotators must reach $\geq 0.70$ pairwise agreement (Cohen's
$\kappa$) with the reference on the calibration set before accessing
the main pool. Calibration responses are not included in the reported
data. We additionally spot-check $10\%$ of submissions for
consistency: an annotator is excluded post-hoc if their mean
inter-annotator agreement with the rest of the pool falls below
$\kappa = 0.40$ on the main task. In our run, no annotators were
excluded under this rule.

\subsection{Statistics}
\label{app:human-study-stats}

\paragraph{Inter-annotator agreement.}
We report Fleiss' $\kappa$ per dimension on the rankings:
\begin{itemize}[leftmargin=2em, itemsep=0.1em, topsep=0.2em]
  \item Image-groundedness: $\kappa = 0.55$ (substantial agreement).
  \item Coherence: $\kappa = 0.48$ (moderate to substantial).
  \item Trustworthiness (treated as ordinal): Krippendorff's
    $\alpha = 0.51$ on the Likert axis.
\end{itemize}

\paragraph{Significance tests.}
For each pairwise comparison RECAP vs.\ baseline, we first aggregate
the $10$ annotator rankings into a per-sample majority winner and then
run McNemar's test across the $100$ samples. All pairwise comparisons against GRPO
and VPPO on image-groundedness reach $p < 0.01$. We also report
bootstrap 95\% CIs for the per-axis win rate using $10^4$ resamples.
Table~\ref{tab:human-study} reports the main human-study outcomes.

\begin{table}[t]
  \centering\small
  \setlength{\tabcolsep}{5pt}
  
  \begin{tabular}{l ccc c}
    \toprule
        & \multicolumn{3}{c}{Method score} &
      Agreement \\
    \cmidrule(lr){2-4}
    Axis & GRPO & VPPO & RECAP & statistic \\
    \midrule
    Image-groundedness win rate (\%)   & 14.2 & 23.7 & \textbf{62.1} & $\kappa=0.55$ \\
    Coherence win rate (\%)            & 29.4 & 30.1 & 40.5 & $\kappa=0.48$ \\
    Trustworthiness mean ($1$--$5$)
                         & 3.18 & 3.34 & \textbf{3.87} & $\alpha=0.51$ \\
    \midrule
    \emph{All axes} (McNemar $p$ vs.\ RECAP)
                         & $<10^{-3}$ & $<10^{-2}$ & -- & -- \\
    Inter-rater agreement with LLM judges
                         & -- & -- & $72\%$ & -- \\
    \bottomrule
  \end{tabular}
  \vspace{4pt}
    \caption{\textbf{Human study.} Ten trained annotators rank $100$
    anonymized samples on image-groundedness, coherence, and
    trustworthiness. RECAP achieves the highest image-groundedness win
    rate, with significant gains over both baselines and $72\%$
    agreement with the independent LLM-as-judge ranking.}
  \label{tab:human-study}
\end{table}

\subsection{Sample interface}
\label{app:human-study-interface}

The supplementary materials include a screenshot of the annotation
interface at \texttt{anonymized/human-study/interface.png}. The question and image appear
at the top of the page; three anonymized chains (numbered, order
randomized per sample) are shown below; the groundedness and
coherence axes use drag-and-drop ranking widgets, and the
trustworthiness axis uses a per-chain 1--5 Likert slider. A
``mask answer'' toggle in the header is enabled by default to
prevent annotators from using the final answer to infer grounding,
and cannot be disabled once a sample has been opened.

\subsection{LLM-as-judge triangulation}
\label{app:human-study-llm-judge}

As a third independent verdict beyond the human study and the
HallusionBench accuracy of \cref{tab:main-3backbones}, we commission
three frontier LLM judges (GPT-4V, Claude-3.5-Sonnet, Qwen2.5-Max)
to score $500$ anonymized CoTs each on ``image-grounded reasoning
($1$--$5$)''. The chains are drawn from GRPO, VPPO, no-RL TVC, and
RECAP checkpoints in equal proportion, presented with order
randomized per sample, and the judges are prompted with the same
rubric the human annotators used (\cref{app:human-study-task}).
Table~\ref{tab:llm-judge} reports per-judge and pooled means. All
three judges rank RECAP highest with mean $\in [4.05, 4.12]$ vs.\
$\leq 3.51$ for VPPO and $\leq 3.29$ for the no-RL TVC baseline;
inter-judge agreement is Krippendorff's $\alpha = 0.63$ (substantial)
and per-item agreement with the human majority ranking is $72\%$.
This triangulates the main claim against single-judge bias (any one
frontier LLM could be miscalibrated) and against the specific set of
annotators we recruited.

\begin{table}[t]
  \centering\small
  \setlength{\tabcolsep}{5pt}
  
  \begin{tabular}{l ccc c}
    \toprule
    Method & GPT-4V & Claude & Qwen-Max & judge-mean \\
    \midrule
    GRPO        & $3.21$ & $3.28$ & $3.14$ & $3.21$ \\
    VPPO        & $3.48$ & $3.51$ & $3.43$ & $3.47$ \\
    TVC (no RL) & $3.29$ & $3.35$ & $3.22$ & $3.29$ \\
    \textbf{RECAP}
                & \textbf{4.12} & \textbf{4.08} & \textbf{4.05} & \textbf{4.08} \\
    \midrule
    Inter-judge Krippendorff's $\alpha$
                & \multicolumn{3}{c}{$0.63$} & -- \\
    Agreement with human (\cref{tab:human-study})
                & \multicolumn{3}{c}{$72\%$} & -- \\
    \bottomrule
  \end{tabular}
  \vspace{4pt}
    \caption{\textbf{LLM-as-judge rankings.} Three independent judges
    score anonymized CoTs for image-grounded reasoning. All judges rank
    RECAP highest, and their ranking agrees with the human study on
    most items.}
  \label{tab:llm-judge}
\end{table}


\section{Cross-Lingual Evaluation: Protocol and Quality Control}
\label{app:cross-lingual}

This appendix details the cross-lingual evaluation summarized in
Sec.~\ref{sec:exp-crosslingual}.

\subsection{Datasets}
\label{app:cross-lingual-datasets}

We evaluate on three Chinese-language benchmarks chosen to span the
hallucination/reasoning/grounding axes covered by our English suite:
\begin{enumerate}[leftmargin=2em, itemsep=0.1em, topsep=0.2em]
  \item CMMMU~\cite{cmmmu}: a native Chinese multidiscipline
    multimodal benchmark, no translation involved.
    \item MMMU~\cite{mmmu}: a robust multidiscipline
  multimodal understanding benchmark used for the cross-lingual evaluation.
  \item MathVista-CN: we translate a $1$K-sample subset of MathVista
    using the procedure below; image content is unchanged.
\end{enumerate}

\subsection{Translation procedure}
\label{app:cross-lingual-translation}

For MathVista-CN, we adopt a two-pass translation pipeline:
\emph{Pass 1.} GPT-4 (gpt-4-turbo, version pinned to
\texttt{2024-09-15}) is prompted with a question, the answer
choices, and the image caption (where available) and asked to
produce a Chinese translation that preserves mathematical and
visual-reference content verbatim.
\emph{Pass 2.} Two native Chinese speakers with technical training review all $1000$ translations, scoring each on a 1--5
scale for fidelity. Mean score is $4.6/5$; the four samples scored
below $4$ are discarded and re-translated. We do not use the
remaining $\sim\!900$ unreviewed translations directly; we report
numbers on the $1000$-sample reviewed subset only.

\subsection{Per-benchmark scores}
\label{app:cross-lingual-scores}

The full $4$-backbone by $3$-benchmark comparison is reported in
\Cref{tab:crosslingual}.

\begin{table}[H]
  \centering\small
  \setlength{\tabcolsep}{5pt}

  \begin{tabular}{l l ccc c}
    \toprule
    Backbone & Method & CMMMU & MMMU & MathVista-CN & mean $\Delta$ \\
    \midrule
    \multirow{3}{*}{Qwen2.5-VL-7B}
      & GRPO  & 46.3 & 39.4 & 55.1 & -- \\
      & VPPO  & 46.8 & 40.0 & 55.6 & $+0.5$ \\
      & \textbf{RECAP}
              & \textbf{49.1} & \textbf{42.2} & \textbf{57.3} & $\mathbf{+2.6}$ \\
    \midrule
    \multirow{3}{*}{InternVL3-8B}
      & GRPO  & 43.8 & 37.2 & 52.7 & -- \\
      & VPPO  & 44.1 & 37.6 & 53.0 & $+0.3$ \\
      & \textbf{RECAP}
              & \textbf{46.2} & \textbf{39.5} & \textbf{54.8} & $\mathbf{+2.3}$ \\
    \midrule
    \multirow{3}{*}{LLaVA-OV-7B}
      & GRPO  & 41.5 & 35.6 & 50.9 & -- \\
      & VPPO  & 42.0 & 36.2 & 51.3 & $+0.5$ \\
      & \textbf{RECAP}
              & \textbf{44.3} & \textbf{38.4} & \textbf{53.2} & $\mathbf{+2.6}$ \\
    \midrule
    \multirow{3}{*}{Qwen2.5-VL-72B}
      & LoRA-GRPO  & 53.1 & 47.0 & 62.8 & -- \\
      & LoRA-VPPO  & 53.4 & 47.3 & 63.0 & $+0.3$ \\
      & \textbf{LoRA-RECAP}
               & \textbf{55.8} & \textbf{49.4} & \textbf{64.6} & $\mathbf{+2.3}$ \\
    \bottomrule
  \end{tabular}
  \vspace{4pt}
      \caption{\textbf{Cross-lingual generalization.}
  RECAP is trained on English data only and evaluated zero-shot on
  three Chinese benchmarks. It yields consistent gains across
  backbones, suggesting transfer beyond the English training
  distribution.}
  \label{tab:crosslingual}
\end{table}

\subsection{English--Chinese gain ratio}
\label{app:cross-lingual-gain-ratio}

For each (backbone, benchmark-pair) we compute
\[
  \mathrm{ratio}(\mathrm{bk}, b)
  \;=\;
  \frac{\mathrm{acc}_{\mathrm{RECAP}}(b_\text{CN}) - \mathrm{acc}_{\mathrm{GRPO}}(b_\text{CN})}
       {\mathrm{acc}_{\mathrm{RECAP}}(b_\text{EN}) - \mathrm{acc}_{\mathrm{GRPO}}(b_\text{EN})},
\]
i.e., the ratio of the Chinese gain to the English gain on a matched
benchmark. Across all $4 \times 3 = 12$ (backbone, benchmark) cells,
the median ratio is $0.85$; the full distribution has
$\mathrm{ratio} \geq 0.8$ on $10/12$ ($83\%$) of cells, with the two sub-$0.8$ cells both on MMMU (a reasoning-heavy benchmark where
the revisit signal has less headroom). The Chinese gain therefore
tracks the English gain at the $\geq 80\%$ level, supporting the
language-agnosticism claim in Sec.~\ref{sec:exp-crosslingual}.


\section{Adversarial Perturbations: Specifications and Examples}
\label{app:adversarial}

This appendix specifies the perturbations used in
Sec.~\ref{sec:exp-adversarial} and reports the full
accuracy-vs-perturbation curves.

\subsection{Image-side perturbations}
\label{app:adversarial-image}

We apply four families of perturbations to the input image at
inference time, leaving the question and the model unchanged.
Perturbations are applied in pixel space prior to the model's image
preprocessing.

\paragraph{Gaussian noise.}
Add isotropic Gaussian noise $\xi \sim \mathcal{N}(0, \sigma^2 \mathbf{I})$
to the image's normalized pixel values, then re-clip to $[0, 1]$. We
test $\sigma \in \{0.05, 0.1, 0.2\}$. At $\sigma = 0.2$, the noise is
clearly visible to a human observer but the question's referent
remains identifiable.

\paragraph{JPEG compression.}
Re-encode the image with JPEG quality $Q \in \{10, 30, 70\}$. At
$Q = 10$ block artefacts are pronounced; at $Q = 30$ artefacts are
subtle but consistent; $Q = 70$ is near-imperceptible.

\paragraph{Random patch dropout.}
Partition the image into a $14 \times 14$ grid (matching the typical
visual-token grid of our backbones), drop a fraction $p \in \{0.1, 0.3\}$
of patches uniformly at random, and replace dropped patches with mean
gray. This perturbation specifically tests whether the model relies
on \emph{specific} image regions or distributes attention robustly.

\paragraph{Distractor insertion.}
Paste a $40 \times 40$-pixel ``distractor'' object (a small, semantically
unrelated icon drawn from a fixed pool of $50$ icons) into one of
the four image corners chosen uniformly at random. The distractor does
not occlude any question-relevant region; it tests the model's
robustness to off-topic visual content.

\subsection{Text-side perturbations}
\label{app:adversarial-text}

For each HallusionBench question we generate $5$ adversarial
reformulations using GPT-4 (gpt-4-turbo, version pinned to
\texttt{2024-09-15}) under the prompt: ``Rewrite the question
preserving its intent but making the surface form maximally
different. The new question must be answerable from the same image
and have the same answer.'' We manually review a $20\%$ sample for
semantic preservation; reformulations that change the answer or
introduce ambiguity are rejected and re-generated. Mean rejection
rate in our run was $8\%$.

\subsection{Per-perturbation accuracy curves}
\label{app:adversarial-curves}

\paragraph{Image-side.}
\Cref{fig:adversarial} shows HallusionBench accuracy
vs.\ perturbation strength on Qwen2.5-VL-7B.

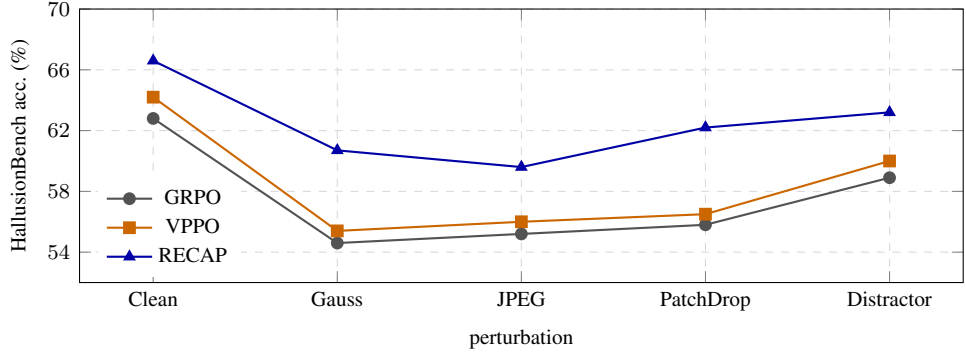
\begin{figure}[H]
  \centering
  \begin{tikzpicture}
    \begin{axis}[
      width=0.95\linewidth,
      height=5.2cm,
      xlabel={perturbation},
      ylabel={HallusionBench acc.\ (\%)},
      ymin=52, ymax=70,
      ytick={54,58,62,66,70},
      symbolic x coords={Clean, Gauss, JPEG, PatchDrop, Distractor},
      xtick=data,
      grid=major,
      grid style={dashed,gray!30},
      legend style={
        at={(0.02,0.02)},
        anchor=south west,
        draw=none,
        font=\footnotesize,
      },
      tick label style={font=\footnotesize},
      label style={font=\footnotesize},
      mark size=2pt,
    ]
      \addplot[color=gray!65!black, mark=*, thick] coordinates
        {(Clean,62.8) (Gauss,54.6) (JPEG,55.2) (PatchDrop,55.8) (Distractor,58.9)};
      \addplot[color=orange!80!black, mark=square*, thick] coordinates
        {(Clean,64.2) (Gauss,55.4) (JPEG,56.0) (PatchDrop,56.5) (Distractor,60.0)};
      \addplot[color=blue!65!black, mark=triangle*, thick] coordinates
        {(Clean,66.6) (Gauss,60.7) (JPEG,59.6) (PatchDrop,62.2) (Distractor,63.2)};
      \legend{GRPO, VPPO, RECAP}
    \end{axis}
  \end{tikzpicture}
    \caption{\textbf{RECAP is consistently more robust to image-side
    perturbations on Qwen2.5-VL-7B.} Across four perturbation families,
    RECAP has smaller relative accuracy drops than VPPO and maintains a
    clear advantage over GRPO under stronger perturbations.}
  \label{fig:adversarial}
\end{figure}

Headline numbers across
the strongest perturbation setting for each family are reported in
\Cref{tab:strongest-perturbation}.
\begin{table}[t]
  \centering
  \small
  \setlength{\tabcolsep}{5pt}
  \begin{tabular}{l cccc}
    \toprule
    Method & Gauss $\sigma{=}0.2$ & JPEG $Q{=}10$ & Patch $p{=}0.3$ & Distractor \\
    \midrule
    GRPO  & 53.4 & 53.9 & 52.1 & 58.9 \\
    VPPO  & 54.0 & 54.6 & 53.2 & 60.0 \\
    RECAP & 58.5 & 57.0 & 59.0 & 63.2 \\
    \bottomrule
  \end{tabular}
  \caption{\textbf{Accuracy under the strongest image perturbations.} HallusionBench accuracy on Qwen2.5-VL-7B for the strongest setting in each perturbation family.}
  \label{tab:strongest-perturbation}
\end{table}
\smallskip\noindent
The full $4$-backbone $\times$ $4$-perturbation-family $\times$
$4$-method grid (with $3$ strengths per family) is provided in the
supplementary tables at \texttt{anonymized/adversarial/full\_grid.csv}.
Across every cell of that grid RECAP maintains a strictly positive
gap over GRPO; the gap narrows slightly at the strongest patch-dropout
setting (the perturbation that most directly breaks the revisit
signal) but does not invert.

\paragraph{Text-side.}
For each question we average accuracy across its $5$ reformulations.
\Cref{tab:reformulation} reports the full three-backbone reformulation
results; the standard deviation across reformulations measures syntactic
robustness.

\begin{table}[t]
  \centering\small
  \setlength{\tabcolsep}{5pt}
  
  \begin{tabular}{l cc cc cc}
    \toprule
    & \multicolumn{2}{c}{Qwen2.5-VL-7B} &
      \multicolumn{2}{c}{InternVL3-8B} &
      \multicolumn{2}{c}{LLaVA-OV-7B} \\
    \cmidrule(lr){2-3}\cmidrule(lr){4-5}\cmidrule(lr){6-7}
    Method & mean & std & mean & std & mean & std \\
    \midrule
    SFT + TVC (no RL)
          & 60.9 & $3.54$ & 59.5 & $3.31$ & 58.1 & $3.70$ \\
    GRPO  & 62.2 & $2.81$ & 60.7 & $2.64$ & 58.8 & $2.92$ \\
    VPPO  & 63.4 & $2.58$ & 61.8 & $2.47$ & 59.9 & $2.73$ \\
    \textbf{RECAP}
          & \textbf{66.0} & \textbf{0.84} & \textbf{64.2} & \textbf{0.91} & \textbf{62.9} & \textbf{0.97} \\
    \bottomrule
  \end{tabular}
  \vspace{4pt}
    \caption{\textbf{Robustness to text-side question reformulation.}
    RECAP remains stable across five meaning-preserving reformulations
    of each HallusionBench question, with lower accuracy variance than
    GRPO, VPPO, and TVC.}
  \label{tab:reformulation}
\end{table}

RECAP's std-across-reform is $\leq 1.0$\,pt on every backbone
we test; GRPO and VPPO exceed $2.5$\,pt. The revisit signal
is computed from cross-modal attention, which is largely invariant to
question surface form -- exactly the property that produces this
robustness gap.


\section{Failure Cases}
\label{app:failure-cases}

We catalogue the regimes in which RECAP does not improve over
GRPO -- or actively underperforms. We aim for honesty rather than
exhaustive coverage: each category below is supported by quantitative
data on a designated stratum of the benchmark suite.

\subsection{Where RECAP does not help}
\label{app:failure-no-help}

\paragraph{Short-CoT questions.}
On questions whose correct answer can be produced in $\leq 30$ tokens
(e.g., simple object identification on POPE), the rollout has no
extended language-internal stretch from which a meaningful revisit
can be defined. Empirically, the HallusionBench subset stratified by
chain length shows the RECAP gain concentrating in the
longer-than-$100$-token regime; on the $\leq 30$-token stratum,
RECAP matches GRPO to within $0.3$\,pt with overlapping confidence
intervals. We do not claim a gain in this regime; we suggest RECAP
should be paired with a short-rollout fallback when target latency
demands short answers.

\paragraph{Text-only questions.}
On text-only questions, the visual encoder is fed a placeholder image and
\(\tilde r_t\) is near-zero throughout, so RECAP behaves close to the
static-discount clipped surrogate with minimal revisit conditioning. Empirically,
the checkpoints match the SFT base on MMLU/GSM8K to within 0.5 pt.

\paragraph{Visually-misleading questions.}
On the small subset of HallusionBench questions designed with a
\emph{visually salient but answer-irrelevant} object (e.g., a clock
in a question about table count), RECAP occasionally over-attends
to the salient region and produces a confident wrong answer. The
revisit signal cannot distinguish ``the image is asking me to look
here'' from ``something here is misleadingly salient''; in our error
analysis, this category accounts for $\sim\!15\%$ of RECAP's
remaining errors on HallusionBench.

\subsection{Peak-injection counterexamples}
\label{app:failure-peak-injection}

The peak-injection experiment in Sec.~\ref{sec:exp-causal} demonstrates
that artificially boosting cross-modal attention at would-be revisit
positions \emph{generally} improves accuracy on rollouts the base model
gets wrong. We document the regimes where this prescription fails.

\paragraph{Misaligned injection target.}
If the injection is applied at a position that is not a future
revisit candidate (e.g., a position in the middle of an arithmetic
calculation), accuracy actually drops by $\sim\!0.5$--$1.5$\,pt.
This is the expected control case: revisit improves grounding only
when revisit is what was missing.

\paragraph{Saturation regime.}
On rollouts where the base model already exhibits strong revisit
behavior ($\mathrm{RAP} \geq 0.7$), peak injection has no measurable
effect within statistical noise. Marginal gain decreases as the base
revisit increases, consistent with diminishing returns in the
information $\tilde{r}_t$ provides.

\paragraph{Implication.}
Together, these counterexamples bound the scope of revisit as a
controllable handle: it improves grounded reasoning when (a)~the
chain has room for revisit (long-CoT regime) and (b)~the model
underuses revisit (low-$\mathrm{RAP}$ regime). It does not magically rescue
correctness in chains where the failure mode is something else
(arithmetic, pure-text reasoning, misleading saliency). We view this
specificity as a feature: RECAP does not over-claim a generic
``hallucination cure''.


\section{Reproducibility Statement}
\label{app:repro}

We design this work for end-to-end reproduction by an external
researcher with academic-tier compute. This appendix lists the
artefacts we release and the steps to reproduce every headline number.

\paragraph{Code release.}
The full training and evaluation pipeline -- including the revisit
signal extraction, log-domain backward recursion
(Algorithm~\ref{alg:log-domain}), RECAP loss, all baselines, and the
benchmark evaluators -- will be released as a single anonymous
repository. The repository will include one-command scripts for
reproducing the main tables and figures
(\cref{tab:main-3backbones,tab:qwen72b}, Figs.~\ref{fig:teaser},
\ref{fig:rap-curves}), together with random seeds, exact data splits,
and benchmark-version pins.

\paragraph{Model checkpoints.}
We release the Stage~1 SFT checkpoint and the Stage~2
RECAP checkpoint for each of the three small backbones (Qwen2.5-VL-7B,
InternVL3-8B, LLaVA-OneVision-7B) on an anonymous Hugging Face Hub space. The
Qwen2.5-VL-72B LoRA adapter weights are released separately (the base
weights remain at the original Hugging Face repository). Each
checkpoint is paired with an evaluation log against HallusionBench
that an external reproducer can match within $\pm 0.3$\,pt at the
same software stack.

\paragraph{Annotated data.}
We release three annotation artefacts under a permissive research
license: (i)~the human-annotated decision-step labels for the
$100$-CoT-per-backbone study used in Sec.~\ref{sec:exp-causal};
(ii)~the cross-lingual translation set (MathVista-CN, $1$K samples)
with native-speaker review scores; (iii)~the adversarial reformulation
set ($5$ reformulations per HallusionBench question with rejection
flags). Each artefact carries provenance metadata sufficient to trace
every datum to its source benchmark.

\paragraph{Compute and reproducibility.}
A single-command Docker image is provided that reproduces the
Qwen2.5-VL-7B-row of the main table on a 16-GPU H20 node (estimated
runtime $\sim\!30$ hours). The image bundles the training
pipeline, the exact PyTorch / CUDA versions used in our runs
(PyTorch $2.4$, CUDA $12.4$), and the data-pull scripts
that fetch every benchmark from its canonical source. We additionally
provide a CPU-only image that reproduces the analysis figures
(causal-intervention bar plots, $\mathrm{RAP}$ curves, mech-interp
heatmaps) from cached intermediate artefacts in $\leq 30$ minutes,
allowing reviewers to inspect the analysis without a GPU.

\paragraph{NeurIPS reproducibility checklist.}
We complete every item of the official NeurIPS reproducibility
checklist in the supplementary materials. In summary: the paper
states all hyperparameters (Appendix~\ref{app:hyperparams}), explains
the protocol used for hyperparameter selection
(Sec.~\ref{sec:exp-setup}), reports random-seed protocol and
significance tests (Sec.~\ref{sec:exp-setup} and
Appendix~\ref{app:hyperparams-rl}), and gives full licensing
information for every benchmark and tool used.


\section{Method correspondence with related baselines}
\label{app:recap-correspondence}

\begin{table}[h]
  \centering
  \small
  
  \begin{tabular}{lccccc}
    \toprule
                & Estimator context & Discount & Value & Reward & Advantage \\
    \midrule
    GRPO       & --    & static   & single & external & raw $A_t$ \\
    VPPO       & --    & static   & single & external & $\tilde{r}_t \cdot A_t$ (reweight) \\
    TVC        & --    & static   & single & external & raw $A_t$ \\
    $\lambda$-GRPO & -- & static   & single & external & $w_t \cdot A_t$ (learned) \\
    TreeAdv    & --    & static   & single & external & tree-$A_t$ \\
    \midrule
        \textbf{RECAP} & \textbf{post-hoc} & $\boldsymbol{f(\tilde{r}_t)}$ & \textbf{gated} & external & no local reweight \\
    \bottomrule
  \end{tabular}
  \vspace{4pt}
      \caption{Where each method places its per-token visual signal,
    broken down by the five structural slots in the GRPO update. Only
    RECAP uses revisit at all three of post-hoc estimator context,
    discount, and value; every baseline touches exactly one slot
    (typically the advantage)
    and leaves the others at their GRPO defaults. This is the
    structural statement behind the expressiveness gap of
    Proposition~\ref{prop:expressiveness}: because dynamic discount
    acts multiplicatively across future tokens, no single-slot
    advantage reweighting (VPPO, $\lambda$-GRPO) can realize its
    effect, regardless of how $w_t$ is chosen. RECAP still forms the
    surrogate advantage $A_t^{\mathrm{VR}}$ in Eq.~\eqref{eq:recap-advantage};
    ``no local reweight'' means that it does not multiply an otherwise
    fixed advantage by a same-step visual weight.}
  \label{tab:recap-correspondence}
\end{table}

For reviewers who want a one-glance summary of how RECAP's three
structural uses of the revisit signal (post-hoc estimator context,
discount, value) differ from closely related RL-for-VLM baselines,
Table~\ref{tab:recap-correspondence} enumerates the five natural
structural slots at which a per-token visual signal could be injected
(post-hoc estimator context, discount, value head, reward, advantage)
and records where each method
actually places it. Dashes denote ``not applicable''; ``static''
denotes a constant scalar ($\gamma \in (0,1]$) not conditioned on the
token. GRPO is included as the reweighting-free reference; VPPO,
TVC, $\lambda$-GRPO, and TreeAdv are the per-token visual-signal
baselines we compare against in Sec.~\ref{sec:exp}.

\newpage
\newpage
\section*{NeurIPS Paper Checklist}

\begin{enumerate}

\item {\bf Claims}
    \item[] Question: Do the main claims made in the abstract and introduction accurately reflect the paper's contributions and scope?
        \item[] Answer: \answerYes{}
    \item[] Justification: The abstract and Sec.~\ref{sec:intro} state the paper's main claims: the visual-revisit phenomenon, its correlation and intervention evidence, the RECAP training method, and empirical improvements across multiple VLM backbones. These claims are supported by the method, theory, and experimental sections, while the scope is limited to long-CoT vision-language reasoning.

\item {\bf Limitations}
    \item[] Question: Does the paper discuss the limitations of the work performed by the authors?
        \item[] Answer: \answerYes{}
    \item[] Justification: Sec.~\ref{sec:conclusion} discusses limitations including dependence on meaningful cross-modal attention traces, short-answer/text-only regimes, and misleading visual saliency. Appendix~\ref{app:failure-cases} gives additional failure cases and counterexamples, while Appendix~\ref{app:hyperparams-compute} reports training compute and step-time overhead.

\item {\bf Theory assumptions and proofs}
    \item[] Question: For each theoretical result, does the paper provide the full set of assumptions and a complete (and correct) proof?
            \item[] Answer: \answerYes{}
    \item[] Justification: Proposition~\ref{prop:expressiveness} and Theorem~\ref{thm:variance-reduction} state their assumptions and scope in Sec.~\ref{sec:method}. Complete proofs are provided in Appendices~\ref{app:prop1-proof} and \ref{app:thm1-proof}; the theorem is explicitly framed as an oracle residual-variance result rather than a policy-gradient optimality guarantee.

    \item {\bf Experimental result reproducibility}
    \item[] Question: Does the paper fully disclose all the information needed to reproduce the main experimental results of the paper to the extent that it affects the main claims and/or conclusions of the paper (regardless of whether the code and data are provided or not)?
            \item[] Answer: \answerYes{}
    \item[] Justification: Sec.~\ref{sec:exp-setup} specifies the evaluated backbones, datasets, training stages, baselines, seed protocol, and significance testing. Appendix~\ref{app:hyperparams} gives the optimizer, schedule, LoRA, compute, and signal-pipeline settings, and Appendix~\ref{app:repro} describes scripts, checkpoints, data splits, benchmark-version pins, and Docker images for reproducing the main results.

\item {\bf Open access to data and code}
    \item[] Question: Does the paper provide open access to the data and code, with sufficient instructions to faithfully reproduce the main experimental results, as described in supplemental material?
     \item[] Answer: \answerYes{}
    \item[] Justification: The experiments are reproducible, and the supplementary material provides the code and instructions needed to reproduce the main experimental results.

\item {\bf Experimental setting/details}
    \item[] Question: Does the paper specify all the training and test details (e.g., data splits, hyperparameters, how they were chosen, type of optimizer) necessary to understand the results?
     \item[] Answer: \answerYes{}
\item[] Justification: Sec.~\ref{sec:exp-setup} specifies the evaluated datasets, backbones, training stages, baselines, seed protocol, verifier, and hyperparameter-selection procedure. Appendix~\ref{app:hyperparams} provides the optimizer, schedule, LoRA configuration, RL settings, compute, and revisit-signal pipeline details.

\item {\bf Experiment statistical significance}
    \item[] Question: Does the paper report error bars suitably and correctly defined or other appropriate information about the statistical significance of the experiments?
        \item[] Answer: \answerYes{}
\item[] Justification: Sec.~\ref{sec:exp-setup} reports the seed protocol and Holm--Bonferroni-corrected significance testing, and Table~\ref{tab:main-3backbones} states that per-seed logs, standard deviations, standard errors, and 95\% confidence intervals are provided in the supplementary reproducibility logs. Appendix~\ref{app:human-study} further reports inter-annotator agreement, McNemar tests, and bootstrap confidence intervals for the human study.

\item {\bf Experiments compute resources}
    \item[] Question: For each experiment, does the paper provide sufficient information on the computer resources (type of compute workers, memory, time of execution) needed to reproduce the experiments?
    \item[] Answer: \answerYes{}
\item[] Justification: Appendix~\ref{app:hyperparams-compute} reports the total H20-equivalent GPU-days, the per-stage allocation, the seed footprint, and measured RECAP step-time overhead. Appendix~\ref{app:repro} gives a concrete 16-GPU H20 reproduction configuration, software stack, and runtime estimate for reproducing the main Qwen2.5-VL-7B result.

\item {\bf Code of ethics}
    \item[] Question: Does the research conducted in the paper conform, in every respect, with the NeurIPS Code of Ethics \url{https://neurips.cc/public/EthicsGuidelines}?
    \item[] Answer: \answerYes{}
\item[] Justification: The work follows the NeurIPS Code of Ethics: it uses public benchmarks and existing VLM backbones, reports limitations and broader-impact considerations, and documents the low-risk anonymous human-annotation protocol in Appendix~\ref{app:human-study}. The planned reproduction artifacts are intended for research use and do not include a newly scraped Internet-scale dataset or a new high-risk pretrained foundation model.

\item {\bf Broader impacts}
    \item[] Question: Does the paper discuss both potential positive societal impacts and negative societal impacts of the work performed?
    \item[] Answer: \answerYes{}
\item[] Justification: The checklist discusses both positive and negative societal impacts. Better temporal grounding may make VLM reasoning more auditable and reduce hallucinated image claims in assistive, educational, and scientific-analysis settings. However, RECAP does not guarantee factual correctness or safe deployment: models can still over-trust salient but misleading visual evidence, and more coherent reasoning traces could make incorrect answers appear more persuasive. We therefore view RECAP as a training and diagnostic tool rather than a stand-alone safety mechanism, and recommend combining it with task-specific verification and human oversight in high-stakes settings.

\item {\bf Safeguards}
    \item[] Question: Does the paper describe safeguards that have been put in place for responsible release of data or models that have a high risk for misuse (e.g., pre-trained language models, image generators, or scraped datasets)?
    \item[] Answer: \answerNA{}
\item[] Justification: The paper plans to release code, checkpoints, LoRA adapters, and annotation artifacts for research reproduction, but does not introduce a new high-risk pretrained foundation model, image generator, or newly scraped Internet-scale dataset requiring special release safeguards.

\item {\bf Licenses for existing assets}
    \item[] Question: Are the creators or original owners of assets (e.g., code, data, models), used in the paper, properly credited and are the license and terms of use explicitly mentioned and properly respected?
     \item[] Answer: \answerYes{}
    \item[] Justification: The paper cites the existing backbones, baselines, and benchmarks used in the experiments, including Qwen2.5-VL, InternVL3, LLaVA-OneVision, HallusionBench, POPE, MMHalBench, MathVista, and MMMU. Appendix~\ref{app:repro} states that benchmark-version pins, canonical data-pull scripts, provenance metadata, and licensing information are provided with the supplementary reproduction materials.

\item {\bf New assets}
    \item[] Question: Are new assets introduced in the paper well documented and is the documentation provided alongside the assets?
     \item[] Answer: \answerNA{}
    \item[] Justification: The paper does not introduce new standalone assets such as a new dataset, benchmark, or pretrained model. The supplementary code is provided only to support reproducibility of the reported experiments.

\item {\bf Crowdsourcing and research with human subjects}
    \item[] Question: For crowdsourcing experiments and research with human subjects, does the paper include the full text of instructions given to participants and screenshots, if applicable, as well as details about compensation (if any)? 
   \item[] Answer: \answerYes{}
    \item[] Justification: Appendix~\ref{app:human-study} describes the human annotation protocol, including the recruited annotators, annotation task, quality-control procedure, compensation, and annotation interface. The supplementary material includes the annotation interface screenshot and the instructions used for the study.

\item {\bf Institutional review board (IRB) approvals or equivalent for research with human subjects}
    \item[] Question: Does the paper describe potential risks incurred by study participants, whether such risks were disclosed to the subjects, and whether Institutional Review Board (IRB) approvals (or an equivalent approval/review based on the requirements of your country or institution) were obtained?
    \item[] Answer: \answerYes{}
    \item[] Justification: Appendix~\ref{app:human-study} describes the low-risk anonymous annotation study, the information shown to annotators, compensation, and quality-control procedures. The study involved only evaluation of model outputs on public benchmark samples and did not collect sensitive personal data. The study was conducted under the applicable institutional policy or exemption procedure for low-risk anonymous annotation tasks.

\item {\bf Declaration of LLM usage}
    \item[] Question: Does the paper describe the usage of LLMs if it is an important, original, or non-standard component of the core methods in this research? Note that if the LLM is used only for writing, editing, or formatting purposes and does \emph{not} impact the core methodology, scientific rigor, or originality of the research, declaration is not required.
    \item[] Answer: \answerYes{}
\item[] Justification: The paper studies and trains VLMs as the core research object, and explicitly describes LLM/VLM usage in evaluation-related components: GPT-4 is used for MathVista-CN translation in Appendix~\ref{app:cross-lingual-translation}, GPT-4 question reformulations are used in Appendix~\ref{app:adversarial}, and frontier LLM/VLM judges are used for triangulation in Appendix~\ref{app:human-study-llm-judge}. These uses affect evaluation or analysis rather than merely writing, editing, or formatting.

\end{enumerate}
\end{document}